\documentclass{article} % For LaTeX2e
\usepackage{iclr2026_conference,times}
\usepackage{amsmath,amsfonts,bm}

\def\eqref#1{equation~\ref{#1}}
\def\1{\bm{1}}

\DeclareMathAlphabet{\mathsfit}{\encodingdefault}{\sfdefault}{m}{sl}
\SetMathAlphabet{\mathsfit}{bold}{\encodingdefault}{\sfdefault}{bx}{n}

\usepackage{amsmath}
\usepackage{amssymb}
\usepackage{mathtools}
\usepackage{amsthm}
\usepackage{algorithm}
\usepackage[endLComment=,italicComments=false]{algpseudocodex}
\usepackage[utf8]{inputenc} % allow utf-8 input
\usepackage[T1]{fontenc}    % use 8-bit T1 fonts
\DeclareTextFontCommand{\texttt}{\fontfamily{cmtt}\selectfont}
\usepackage{hyperref}       % hyperlinks
\usepackage{url}            % simple URL typesetting
\usepackage{booktabs}       % professional-quality tables
\usepackage{amsfonts}       % blackboard math symbols
\usepackage{nicefrac}       % compact symbols for 1/2, etc.
\usepackage{microtype}      % microtypography
\usepackage{float}
\usepackage{enumitem}
\usepackage[font=small]{caption}
\usepackage{amsmath}
\usepackage{mathtools}
\usepackage{textcomp}
\usepackage{multirow}
\usepackage{booktabs}
\usepackage{boldline}
\usepackage{tcolorbox}
\usepackage[dvipsnames]{xcolor}
\usepackage[table,xcdraw,dvipsnames]{xcolor}
\usepackage{wrapfig}
\usepackage{threeparttable}
\usepackage[thicklines]{cancel}
\usepackage{setspace}
\usepackage{amsthm}
\usepackage{natbib}
\usepackage{cleveref}
\usepackage{makecell}
\usepackage{tablefootnote}
\setcitestyle{authoryear,round}
\usepackage{subcaption}  
\usepackage{pythonhighlight}
\usepackage{mathtools}

\usepackage{pgfplots}
\pgfplotsset{compat=1.18,
every axis/.append style={
    scale only axis,
    enlargelimits=false,
    ylabel near ticks,
    xlabel near ticks,
    axis line style={draw=none},
    title style={yshift=-8pt, font=\scriptsize},
    label style={font=\scriptsize},
    ylabel style={yshift=-4pt},
    xlabel style={yshift=4pt},
    tick label style={font=\tiny},
    x tick label style={yshift=2pt},
    y tick label style={xshift=2pt},
    ytick style={draw=none},
    ytick align=inside,
    }
}
\usepackage{tikz}
\usetikzlibrary{positioning,shapes,fit,calc,arrows.meta,patterns}

\theoremstyle{plain}
\newtheorem{theorem}{Theorem}

\theoremstyle{definition}

\theoremstyle{remark}

\newcommand{\methodname}{VGF}
\definecolor{mycolor}{rgb}{0.29, 0.30, 0.64}

\definecolor{sb_blue}{RGB}{31,119,180}
\definecolor{sb_orange}{RGB}{255,127,14}
\definecolor{sb_green}{RGB}{44,160,44}
\definecolor{sb_red}{RGB}{214,39,40}
\definecolor{sb_purple}{RGB}{148,103,189}
\definecolor{sb_brown}{RGB}{140,86,75}
\definecolor{sb_pink}{RGB}{227,119,194}
\definecolor{sb_gray}{RGB}{127,127,127}
\definecolor{sb_yellow}{HTML}{bcbd22}
\definecolor{sb_cyan}{RGB}{23,190,207}

\definecolor{light_gray}{RGB}{242,242,242}
\definecolor{light_mycolor}{RGB}{228,223,240}  % light version of (74,77,163)
\hypersetup{colorlinks=true,urlcolor=mycolor,citecolor=mycolor,linkcolor=mycolor}

\title{\resizebox{\textwidth}{!}{Reinforcement Learning via Value Gradient Flow}}

\author{%
Haoran Xu\thanks{Equal contribution.}\:\:$^{1}$ \quad Kaiwen Hu\footnotemark[1]\:\:$^{2}$ \quad Somayeh Sojoudi$^{2}$ \quad Amy Zhang$^{1}$ \\
$^{1}$University of Texas at Austin \quad $^{2}$University of California, Berkeley  \\
}

\iclrfinalcopy % Uncomment for camera-ready version, but NOT for submission.
\begin{document}

\maketitle

\begin{abstract}
We study behavior-regularized reinforcement learning (RL), where regularization toward a reference distribution (the dataset in offline RL or the base model in LLM RL finetuning) is essential to prevent value over-optimization caused by erroneous out-of-distribution extrapolation. Existing methods either rely on reparameterized policy gradient, which are difficult to scale to large generative models, or on reject sampling, which can be overly conservative when attempting to move beyond the behavior support. In this paper, we propose Value Gradient Flow (VGF), a scalable new paradigm for behavior-regularized RL. VGF casts behavior-regularized RL as an optimal transport problem that maps the reference distribution to the value-induced optimal policy distribution. We solve this transport problem via discrete gradient flow, where value gradients guide particles initialized from the reference distribution. Our analysis shows that VGF imposes regularization implicitly by controlling the transport budget. VGF eliminates explicit policy parameterization while remaining expressive and flexible, this enables adaptive test-time scaling by adjusting the transport budget. Extensive experiments demonstrate that VGF significantly outperforms prior methods, achieving state-of-the-art results on offline RL benchmarks (D4RL, OGBench) and LLM RL tasks.

\textbf{Code and runs can be found at \url{https://ryanxhr.github.io/vgf}}
\end{abstract}

\section{Introduction}

\begin{figure}[h]
\centering
\includegraphics[width=1.0\textwidth]{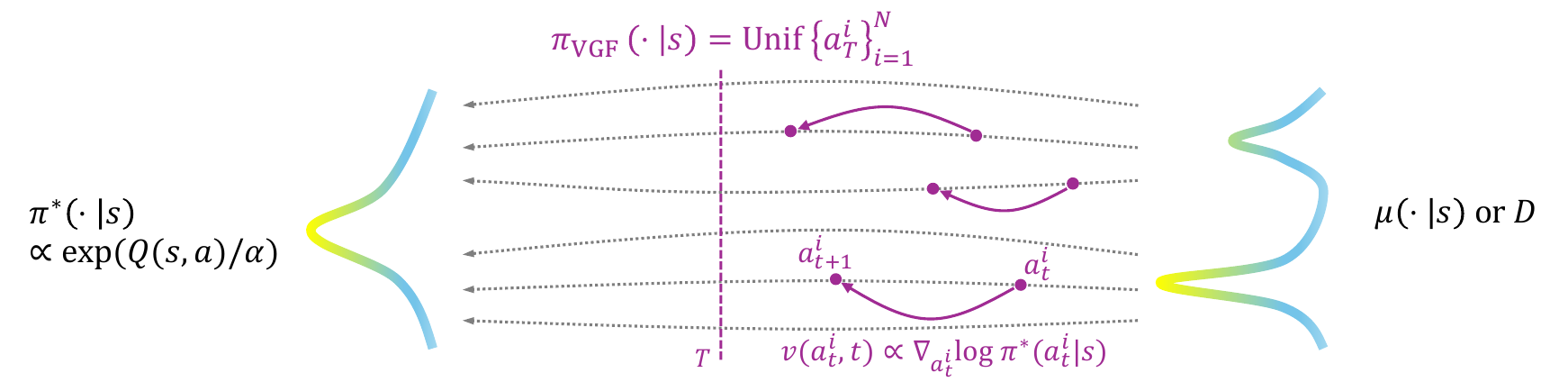}
\caption{\footnotesize 
\textbf{VGF: Value Gradient Flow}. 
VGF reframes behavior-regularized RL as an optimal transport from the behavior distribution towards the Boltzmann value distribution, with the transport budget as implicit regularization. This scales to large generative models and enables adaptive test-time scaling.
}
\label{vgf_illustration}
\end{figure}

% Reinforcement learning (RL) has achieved impressive results across a wide range of domains, from playing complex games like Go~\citep{mnih2013playing,silver2017mastering} to robotic manipulation tasks~\citep{levine2016end}. 
% However, we need a behavior regularization (i.e., KL divergence) in many cases, such as online RL, offline RL, RLHF. 
% For example, in safety-critical or high-cost domains, such as healthcare~\citep{gottesman2018evaluating} and industrial control~\citep{zhan2022deepthermal,zhan2025data}, direct online exploration can be prohibitively risky or expensive, necessitating offline RL approaches that leverage pre-collected datasets without further environment interaction. Recent advances in offline RL have introduced conservative algorithms~\citep{kumar2020conservative,fujimoto2019off} and regularization-based methods~\citep{wu2019behavior,xu2023offline} designed specifically to mitigate extrapolation errors arising from dataset distribution shifts. 
Reinforcement learning (RL) has provided a powerful framework for solving sequential decision-making problems in complex environments. These methods have been successfully applied in diverse domains, ranging from robotics~\citep{levine2016end} to game playing~\citep{mnih2013playing,silver2017mastering}, and have recently become instrumental in fine-tuning large language models (LLMs) to align with human preferences and instructions~\citep{ouyang2022training} and enhancing the reasoning capabilities of LLMs~\citep{shao2024reasoning,guo2025large}. While these successes highlight the broad potential of RL, they also expose a key challenge: policies must often be regularized toward a reference distribution to remain stable and reliable. This challenge arises in both offline RL, where erroneous extrapolation beyond fixed datasets can cause severe value overestimation, and RLHF, where deviating too far from the supervised policy risks reward hacking. In these settings, na\"ive value maximization alone is insufficient. 
% Thus, behavior regularization serves as a crucial mechanism to ensure learning occurs only within the support of trusted behaviors. 
As a result, recent research in both offline RL~\citep{kumar2020conservative,fujimoto2019off,wu2019behavior,xu2023offline} and RLHF~\citep{ouyang2022training,zheng2023reward} has increasingly converged on the paradigm of \textit{behavior-regularized RL}, which balances value maximization with adherence to reliable reference distributions.

The most common approach to behavior-regularized RL is to add explicit divergence or distance penalties (e.g., L2 distance or KL divergence) to the RL learning objective~\citep{ouyang2022training,touvron2023llama,gao2023scaling}.
% This approach restricts policy updates to remain close to the reference distribution, but it introduces several limitations. 
While this constrains policies to remain close to the reference distribution, it also introduces several limitations. 
First, explicit penalty methods typically use a single coefficient to regularize both value learning and policy improvement, even though these two components can prefer different regularization strengths. Tying them together makes the penalty coefficient difficult to tune and can lead to either overly conservative updates or insufficient control of out of distribution drift~\citep{wu2025invisible,korbak2022rl}.
% First, jointly optimizing for reward maximization and distributional proximity creates an optimization conflict, often leading to unstable training and overly conservative policies (e.g., stay within the support of the reference distribution~\citep{wu2025invisible,korbak2022rl}). 
Second, scaling these methods to expressive generative policies such as diffusion~\citep{song2020score,ho2020denoising} and flow models~\citep{lipman2022flow} is challenging. Computing policy gradient requires differentiating through multi step sampling procedures, which is unstable and computationally expensive, while distillation the multi step policy into a one step model can sacrifice expressivity~\citep{park2025flow}.

\noindent \textbf{Behavior-regularized RL without regularization.\footnote{Although \methodname{} uses transport budget as an implicit regularization, here we use "without regularization" to emphassize that VGF doesn't involve any auxiliary regularization during optimization.}} \
We introduce \textbf{Value Gradient Flow (\methodname{})}, which casts behavior-regularized RL as the optimal transport from an (estimated) reference distribution to the optimal policy distribution induced by the value function. 
% where we find the transport budget naturally serves as an implicit behavior regularization.
% which reframes behavior regularization as a bounded transport from an (estimated) reference distribution to regions preferred by the learned value function. 
VGF adds no explicit distance or divergence penalty and does not rely on a parameterized policy.
% Instead of adding an explicit distance/divergence penalty or maintaining a separate parameterized policy, 
Instead, it starts from samples of the reference distribution and gradually nudges them toward higher value regions using a small, fixed number of guidance steps. The distribution after these steps serves as an implicit policy. The transport budget itself (i.e., how far and how often we move) acts as an implicit behavior regularization, limiting deviation from the reference distribution during training while preserving flexibility to enable adaptively scaling at inference. 
Notably, VGF supports different levels of regularization during training and inference. When the inference transport budget is zero, \methodname{} reduces to reject sampling methods, Conversely, using a larger transport budget at inference than during training brings test-time-scaling improvements.
% \methodname{} removes the need to balance the optimization conflict between reward maximization and derivation penalties, avoids brittle coefficient hyperparameter tuning, and better preserves multimodal structure inherited from the reference distribution.  
% \methodname{} also enables test-time scaling by simply increasing the number of guidance steps. 
In RLHF, \methodname{} yields inference-time control that steers a supervised finetuned policy using first-order gradient, sidestepping RL-style optimization and reducing both compute and engineering complexity. 
Across extensive experiments, \methodname{} consistently outperforms strong behavior-regularized baselines, attaining state-of-the-art results on standard offline RL suites (D4RL, OGBench) and delivering substantial gains on RLHF tasks.

% To summarize, the contributions of this paper are as follows:
% % 1. Totally new paradiagm for behavior regularized RL while offering several benefits.
% % 2. The first method that can be used for both offline RL and LLM RL finetuning, with breakthrough performance.
% \begin{itemize}[leftmargin=*]
% \item A new paradigm that casts behavior-regularized RL as bounded transport from a reference distribution to value-preferred regions, yielding an implicit policy
% % with multimodal expressivity and test-time scaling ability.
% with reduced conservatism.
% % \item A simple and stable training recipe that \textbf{decouples} critic learning from policy parameterization and supports \textbf{test-time scaling} without retraining.
% % \item A theoretical perspective showing that the transport budget induces \textbf{implicit regularization}, controlling deviation from trusted behaviors while avoiding over-conservatism.
% \item An unified and scalable framework for both offline RL and LLM RL, offering several distinctive benefits and achieving state-of-the-art empirical performance. 
% % on offline benchmarks and strong improvements in LLM alignment.
% \end{itemize}

\section{Preliminaries}
\noindent \textbf{MDP and value functions.} \
We consider the RL problem presented as a Markov Decision Process (MDP)~\citep{sutton1998introduction}, which is specified by a tuple $\mathcal{M}=\langle \mathcal{S}, \mathcal{A}, \mathcal{P}, d_{0}, r, \gamma \rangle$. Here $\mathcal{S}$ and $\mathcal{A}$ are state and action space, $\mathcal{P}(s'|s, a)$ and $d_{0}$ denote transition dynamics and initial state distribution, $r(s, a)$ and $\gamma$ represent reward function and discount factor, respectively. The goal of RL is to find a policy $\pi(a|s)$ which maximizes expected return $J(\pi) = \mathbb{E}_{\pi}[\sum_{t=0}^{\infty}\gamma^{t} \cdot r(s_t, a_t)]$. In the offline setting, interaction with the environment is prohibited and one needs to learn an optimal $\pi$ from a static replay buffer $\mathcal{D}=\{s_i, a_i, r_i, s^{\prime}_i\}_{i=1}^{N}$ collected from unknown policies. The dataset can be heterogeneous and suboptimal, we denote the 
% empirical visitation distribution of $\mathcal{D}$ as $d^\mathcal{D}$ and the 
empirical behavior policy of $\mathcal{D}$ as $\pi_{\mathcal{D}}$, which represents the conditional distribution $p(a|s)$ observed in the dataset.
% We also refer to $\mathcal{D}$ as the online replay buffer that is updated by filling in new transitions in the online or offline-to-online setting.
% and denote the empirical discounted state-action visitation distribution of $\pi^\mathcal{D}$ as $d^\mathcal{D}$.
% \noindent \textbf{Value functions and visitation distributions.} \quad

RL methods based on approximate dynamic programming typically maintain an action-value function ($Q$-function) and, optionally, a state-value function ($V$-function), referred to as $Q(s, a)$ and $V(s)$ respectively \citep{haarnoja2017reinforcement,nachum2017bridging,kumar2020conservative,kostrikov2021iql}.
Define $Q^\pi: \mathcal{S} \times \mathcal{A} \rightarrow  \mathbb{R}$, where $Q^\pi(s, a) = \mathbb{E}_{\pi}\left[ \sum_{t=0}^\infty \gamma^t r(s_t, a_t) | s_0 = s, a_0 = a \right]$.
The value function is learned by satisfying single-step Bellman consistencies.
% The visitation distribution $d^{\pi}$ is defined as $d^{\pi}(s, a) = (1-\gamma)\sum_{t=0}^{\infty}\gamma^{t}\operatorname{Pr}\left(s_t=s, a_t=a \mid s_0 \sim d_0, \forall t, a_t \sim \pi\left(s_t\right), s_{t+1} \sim \mathcal{P}\left(s_t, a_t\right)\right)$, which measures how likely $\pi$ is to encounter $(s,a)$ when interacting with the environment, averaging over time via $\gamma$-discounting.
% Let $V^*$, $Q^*$ and $d^*$ denote the value functions and visitation distribution corresponding to the regularized optimal policy $\pi^*$. 
Let $\mathcal{T}^{\pi}$ be the Bellman operator with policy $\pi$ such that $(\mathcal{T}^{\pi} Q)(s, a) := r(s, a) + \gamma \mathbb{E}_{s'|s,a}\mathbb{E}_{a'\sim \pi}\left[Q(s',a')\right]$.
Then $Q$ are learned by $\min_{Q} J(Q)=\frac{1}{2} \mathbb{E}_{(s, a) \sim \mathcal{D}}\left[(\mathcal{T}^{\pi} Q - Q)(s, a)^2 \right]$.

\noindent \textbf{Behavior-regularized RL.} \
% Connections and differences (prior, regularization, policy parameteration.) in Offline RL, RLHF.
In general, behavior-regularized RL considers the following constraint optimization problem with a reference distribution/policy $\mu$:
\begin{equation}
\label{eq_behavior_regularized_rl}
\pi^{\ast} = \arg\max\limits_{\substack{\pi}} \ \mathbb{E}_{s \sim \mathcal{D}, a \sim \pi(\cdot|s)} \big[R(s, a) \big] \ \ \text{s.t.} \ \ \mathbb{E}_{s \sim \mathcal{D}} \big[ M \left(\pi(\cdot|s), \mu(\cdot|s) \right) \big] \leq \epsilon,
\end{equation}
where $R(s,a)$ is a differentiable function and $M$ is some distance or divergence measure (e.g., KL-divergence, $L_2$ distance). 
$\mu$ could be the offline data distribution $\pi_{\mathcal{D}}$ in offline RL or the base model after pretraining in LLM RL. $R(s,a)$ could be either (multi-step) $Q$-function or (single-step) reward model in RL from human feedback (RLHF)~\citep{bradley1952rank}.
We now briefly discuss several different approaches to \Cref{eq_behavior_regularized_rl} and their limitations.
% In offline RL we choose $\mu$ as $\pi_{\mathcal{D}}$ and $R(s,a)$ to be the $Q$-function, whereas in RLHF $\mu$ is the LM after supervised finetuning and $R(s,a)$ is set to be the reward model trained on pairwise preference data~\citep{bradley1952rank}.
% Depending on $\pi$ and $M$, several different approaches can be used to extract the optimal policy based on \Cref{eq_behavior_regularized_rl}.

\textbf{(1) Policy gradient with $\pi$ reparameterized.} \
The most straightforward approach is to guide the policy to directly maximize function $R$ with reparameterized gradients. The constraint term will be added as a regularization term with a coefficient $\beta$ to balance these two gradients. 
\begin{equation}
\label{eq_pg}
\max\limits_{\substack{\pi}} \ \mathbb{E}_{s \sim \mathcal{D}} \Big[ \mathbb{E}_{ \color{mycolor} \boldsymbol{a \sim \pi} } \big[R(s, a) \big] - \beta \cdot M\big(\pi(\cdot|s), \mu(\cdot|s) \big) \Big].
\end{equation}
Reparameterized policy gradient is commonly used in offline RL with Gaussian policies~\citep{wu2019behavior,fujimoto2021minimalist,tarasov2024revisiting}.
However, extending this approach to large generative policies such as diffusion and flow matching models is challenging. These models generate actions through an iterative denoising process~\citep{lipman2022flow}. Reparameterized policy gradients require backpropagating through the sampling steps, which is often unstable and computationally expensive~\citep{wang2022diffusion}. An alternative is to use distillation to compress the multi step policy into a one step model~\citep{ding2023consistency,park2025flow}, but this will reduce expressivity.
% One issue of this type of method is that it uses first-order gradient information from $R(s,a)$ where actions are sampled from $\pi$, 
% this may bring over-estimation errors to both the policy and value learning learning without an appropriately selected penalty strength $\beta$.
% this brings over-estimation errors to the policy update without an appropriately selected $\beta$, and further affects value learning due to the \textit{deadly-triad} issue~\citep{van2018deep}.
% Another issue is that computing $\nabla_{\theta} \ \mathbb{E}_{a \sim \pi_\theta} \left[ R(s, a) \right]$ is unstable, requiring backpropagation through time when $\pi$ is a generative model, and cannot be applied to the LLM setting due to the discrete nature of language generation~\citep{rafailov2023direct}.

\textbf{(2) Reject Sampling with $M=\mathrm{KL}$.} \ 
Using KL-divergence gives \Cref{eq_pg} a closed-form solution which can be optimized by doing weighted behavior cloning (BC)~\citep{peng2019advantage,xu2023offline} where actions are sampled from the reference policy, as follows.
\begin{equation}
\label{eq_wbc}
\max_{\pi} \ \mathbb{E}_{s \sim \mathcal{D}, \color{mycolor} \boldsymbol{a \sim \mu} } \big[ \exp(R(s,a)/\beta) \cdot \log\pi(a|s)\big].
\end{equation}
Although simple and easy to implement, using weighted BC tends to be mode-covering, only amplifying weak signals from the reference distribution without extracting new skills or knowledge~\citep{wu2025invisible}.
In fact, a simple best-of-$N$ sampling policy~\citep{nakano2021webgpt}, where $N$ i.i.d. samples are drawn from the reference distribution and one with the highest $R(s,a)$ is returned, is theoretically near optimal for this KL-constrained RL problem~\citep{beirami2024theoretical,yang2024asymptotics}.
\begin{equation}
\label{eq_bon}
\pi^{*} = {\arg\max}_{\substack{ \color{mycolor} \boldsymbol{a_i \sim \mu, i \in [N] }}} \  R(s, a_i).
\end{equation}

\section{Value Gradient Flow}\label{sec:vgf}
VGF is designed to solve the above-mentioned challenges and provide a \textbf{simple and scalable} solution to \Cref{eq_behavior_regularized_rl}, and we give a detailed introduction in this section. 

\subsection{Behavior-Regularized RL as Optimal Transport}
We first consider a surrogate optimization objective that augments the value maximization objective in \Cref{eq_behavior_regularized_rl} with a policy entropy maximization term: $\mathbb{E}_{ a \sim \pi } \left[R(s, a) \right] + \alpha H(\pi(\cdot | s))$, where $H(\pi(\cdot | s)) \triangleq \mathbb{E}_{\pi}\left[ -\log \pi(a | s) \right]$ is the causal entropy of the policy $\pi$ at state $s$. This Maximum-Entropy (MaxEnt) formulation of RL is well-known to enhance the exploration and robustness of the policy~\citep{haarnoja2018soft,garg2021iq,eysenbach2021maximum}. However, our intuition here is that optimizing this MaxEnt objective turns the optimal policy distribution from greedy max to softmax over the whole action space, resulting a variational distribution
as the Boltzmann distribution over the value function $R(s,a)$~\citep{ziebart2010modeling,bloem2014infinite}:
\begin{align}
\label{eq_maxent}
\pi_{R}^*(a|s) = \frac{1}{Z_s}\exp{(R(s, a) / \alpha)},
\end{align}
where $Z_s$ is the normalization factor given as $\sum_{a^{\prime}} \exp \left(R\left(s, a^{\prime}\right) / \alpha \right)$.

\noindent \textbf{Particle-based gradient flow.} \
% We view behavior regularization as transporting probability mass from $\mu$ to $\pi^\star_Q$.
We reframe the value-maximization problem %%AZ.9.24: what is it? HX: revised
as an optimal transport problem that transports probability mass from  distribution $\mu$ to distribution $\pi^{*}_{R}$ defined in \Cref{eq_maxent}.
A natural way to formalize this transport is as a gradient flow of the functional $F(q)=\mathrm{KL}(q\,\|\,\pi^\star_R)$ on the space of probability measures endowed with the Wasserstein metric~\citep{jordan1998variational,ambrosio2008gradient,peyre2019computational}.
The resulting continuous-time evolution $q_t$ follows the continuity equation
$\partial_t q_t + \nabla\!\cdot(q_t v_t)=0$
with the steepest-descent velocity field
$v_t=\nabla \log \pi^\star_R-\nabla \log q_t$,
so that $F(q_t)$ decreases monotonically and the stationary distribution is $\pi^\star_R$.
% In our RL setting, the target score simplifies to $\nabla_a \log \pi^\star_Q(a\mid s)=(1/\alpha)\,\nabla_a Q(s,a)$, revealing a direct coupling between transport and value gradients.
% \paragraph{A discrete scheme and a particle approximation.}
However, directly solving $q_t$ is intractable, so we adopt the Jordan-Kinderlehrer-Otto (JKO) minimizing-movement scheme~\citep{jordan1998variational} to obtain a discrete gradient flow:
\begin{equation}
q_{k+1} = \arg\min_{q} \ \mathrm{KL}\!\left(q\,\|\,\pi^\star_R\right) + \frac{1}{2h}\,W_2^2\!\left(q,\,q_k\right),
\label{eq_jko}
\end{equation}
where $h>0$ is the step size and $W_2$ is the $2$-Wasserstein distance~\citep{peyre2019computational}.
In Euclidean space, \Cref{eq_jko} reduces to gradient descent on the function landscape. However, it is intractable as, in general, $q_k$ is infinite-dimensional.
% This motivates a kernelized, particle-interacting approximation that enforces both attraction toward high-probability regions of $\pi^\star_Q$ and repulsion to maintain diversity—precisely the mechanism provided by Stein variational methods.

% \paragraph{SVGD as a particle-based WGF solver.}
To obtain a practical solver, we approximate $q_k$ by an empirical measure over $N$ particles in action space (for a fixed state $s$), $q_k \approx \tfrac{1}{N}\sum_{i=1}^N \delta_{a_i^{(k)}}$, and seek an update rule for $\{a_i^{(k)}\}_{i=1}^N$ that decreases \Cref{eq_jko}.
By restricting the velocity field $v$ to the unit ball of a vector-valued reproducing kernel Hilbert space (RKHS), we get a solver that can be derived as the nonparametric functional gradient method that most rapidly decreases $\mathrm{KL}(q\,\|\,\pi^\star_R)$ within the RKHS~\citep{liu2016stein,liu2017stein}. This yields a particle-based gradient flow solver that approximates the discrete gradient flow as $a_i^{\,(l+1)} = a_i^{\,(l)} + \epsilon \cdot \phi(a_i^{\,(l)})$, where
\begin{equation}
% &a_i^{\,(l+1)} = a_i^{\,(l)} + \epsilon \cdot \phi \big(a_i^{\,(l)}\big), \\
\phi(x) = \frac{1}{N}\sum_{j=1}^N \Big[k(a_j,x)\, \underbrace{ \nabla_{a_j} \log \pi^\star_R(a_j | s)}_{=\nabla_{a_j} R(s,a_j) / \alpha}  +\nabla_{a_j} k(a_j,x) \Big] 
\xrightarrow{\scriptsize\text{w/o MaxEnt}} \frac{1}{N}\sum_{j=1}^N k(a_j,x) \nabla_{a_j} R(s,a_j).
\label{eq_svgd}
\end{equation}
Here, $\epsilon$ is the step size and $k(\cdot,\cdot)$ is the kernel function. 
% (e.g., RBF kernel with $k(a_i,a_j) = \exp(||a_i -a_j||^2/2 \sigma^2)$. 
The first term in $\phi(a_i)$ drives the particles toward the high probability regions of $\pi^{*}_R$ (i.e., with high $R(s,a)$), while the second term serves as a repulsive force to encourage dispersion and preserve multi-modality of the particles. The second term and $\alpha$ vanish when we transform the surrogate MaxEnt objective back to the original objective. This is equivalent to taking the limit $\alpha \to 0$ and absorbing $\alpha$ into the step size $\epsilon$.

% Convergence of this particle-based gradient flow to the target distribution can be easily proved as $\mathbb{E}_{a^L \sim \pi^*_R(\cdot|s)}[\phi(a^{L})]=0$.
% We would like to show the advantage of optimizing the Wasserstein distance via SVGD over reducing KL divergence. An simple observation is that if the supports of the two distributions are different, the KL divergence tends to be infinite, especially when there exists an action $a \in S_\pi, a \not\in S_\mu$. However, the Wasserstein distance is always a finite value, regardless of the difference of supports. We can establish a theorem to quantify the Wasserstein distance.

Note that an \textbf{implicit behavior regularization} is imposed via controlling the transport budget ($L$, $\alpha$ and $\epsilon$). Intuitively, \Cref{eq_svgd} performs a kernel-smoothed transport to each particle in the action space, resulting in a controlled derivation from the reference distribution. Theoretically, we show in the following that the Maximum Mean Discrepancy (MMD) distance between the initial particles sampled from the reference policy and the particles generated by \methodname{} is bounded.
% \haoran{can we prove that the distance between $a^{L}$ an $a^{0}$ is bounded?}
% \haoran{could diffusion models do similar things}?
% \amy{If you are controlling the distance in the action space it requires a Lipschitz assumption on the dynamics (which does not always hold true in the real world, but I think is an ok assumption to make}
% \begin{lemma}
%     If the value function $Q(s,a)$ is $c$-Lipschitz with regard to the action $a$, then every dimension of the update function $h$ (i.e., $\|h\|_\infty$) is upper-bounded by $\frac{c}{\alpha}+\frac{1}{\sigma\sqrt{e}}$.
% \end{lemma}
\begin{theorem}
    Assume the value function $R(s,a)$ is $c$-Lipschitz w.r.t the input action $a$. Define the implicit policy that performs \Cref{eq_svgd} for $L$ steps with $N$ particles as $\pi_{N}^{L}$. We have
    \begin{equation*}
        \mathrm{MMD}^2(\mu, \pi_{N}^{L}) = \mathrm{MMD}^2(\pi_{N}^{0}, \pi_{N}^{L}) \leq \frac{2\epsilon L}{\sigma\sqrt{e}}\left(\frac{c}{\alpha} + \frac{1}{\sigma\sqrt{e}}\right).
    \end{equation*}
    % With $L$ iterations of SVGD updates, the target distribution $\mu^L$ has a bounded distance from the behavioral distribution $\mu^0$ of $2\epsilon KHL$, where $\epsilon$ is the update step size, $K:=\frac{1}{\sigma\sqrt{e}}$, and $H:=\frac{c}{\alpha} + \frac{1}{\sigma\sqrt{e}}$.
    \label{thm:mmd_bound}
\end{theorem}

% Value function is updated via the conventional TD learning with the target value averaged over all particles from the implicit policy.
% \begin{equation}
% Q(s_t, a_t) \leftarrow r(s_t, a_t) + \gamma \mathbb{E}_{s_{t+1}, a^{1}_{t+1} \sim \mu_1} \left[ Q(s_{t+1}, a^{1}_{t+1}) \right],
% \end{equation}
% Why don't use maximum entropy RL formulation.

% \subsection{VGF in the RLHF setting}
\noindent \textbf{VGF in the LLM setting.} \
In the LLM RL setting, at time step $t$, the action $a_t$ is a discrete token and the state $s_t$ is the token sequence
$s_t = (x_0,\ldots,x_L, a_0,\ldots,a_{t-1})$, where $x=(x_0,\ldots,x_L)$ is the input prompt and
$y=(a_0,\ldots,a_{t-1})$ are the generated tokens up to step $t-1$.
The transition function $P$ updates the state deterministically via concatenation:
$s_{t+1} = P(s_t,a_t) = s_t \,\Vert\, a_t$.

A direct application of Eq.~(\ref{eq_svgd}) to tokens is infeasible because tokens are discrete. We therefore perform \methodname{} in a \textbf{continuous surrogate space} and decode back to the discrete token space only at the end of the gradient flow. Let $u$ be a differentiable representation of a full response $y$. The representation could either be the token-embedding matrix $u\in\mathbb{R}^{T\times d}$ or a latent vector $u=z\in\mathbb{R}^{m}$ of a flow or diffusion language model with $y=\mathrm{Dec}(z)$. Denote $y_i^{(l)}=\mathrm{Dec}(u_i^{(l)})$. 
% Although we cannot directly optimize parameter-level gradients $\nabla_\theta \mathbb{E}_{y\sim\pi_\theta(\cdot\mid x)}[R(x,y)]$ due to discrete autoregressive sampling, the reward model yields response-level gradients $\nabla_y R(x,y)$ that are back-propagated to the embedding surrogate via chain rule:
Because the reward model is differentiable with respect to its input embeddings, response-level gradient $\nabla_y R(x,y)$ can be back-propagated to the surrogate via the chain rule as follows.
\begin{equation}
\label{eq_svgd_llm}
\nabla_{u_i}\log \pi_R^*\big(y_i^{(l)} | x \big)
= \frac{1}{\alpha}\, J_i^{\top}\,\nabla_{y} R\big(x, y_i^{(l)}\big),
\quad
J_i := \frac{\partial \mathrm{Dec}\big(u_i^{(l)}\big)}{\partial u_i^{(l)}}.
\end{equation}
% After $L$ steps we decode $y_i^{(L)}=\mathrm{Dec}(u_i^{(L)})$ and return the set $\{y_i^{(L)}\}_{i=1}^{N}$.
One motivation to use \methodname{} is that the SFT policy is far from random, with most probability mass concentrating on a small subset of tokens and modes. 
% VGF exploits this by starting from plausible responses sampled from $\mu$ and performing \emph{local}, first-order transport guided by $\nabla_y R(x,y)$. 
Note that VGF utilizes first-order gradient guidance from $R$, this avoids high-variance PPO-style optimization~\citep{ouyang2022training} and enables inference-only control similar to best-of-$N$ sampling. 
However, the difference is that \methodname{} steers particles toward high-reward modes, the resulting implicit policy need not remain within the support of the reference distribution, as shown in the following theorem.
\begin{theorem}
\label{thm:support_shift}
Define the $\epsilon$-support of a distribution $P$ as $\mathrm{supp}_\epsilon(P) := \{x:p(x) \geq \epsilon\}$. We have
\begin{equation*}
    \mathrm{supp}_\epsilon(\pi_{N}^{L}(\cdot|s)) \not\subseteq \mathrm{supp}_\epsilon(\mu(\cdot|s)).
\end{equation*}
\end{theorem}
This indicates that, unlike the methods discussed in Section~2 which result in $\mathrm{supp}_\epsilon(\pi(\cdot|s)) \subseteq \mathrm{supp}_\epsilon(\mu(\cdot|s))$~\citep{wu2025invisible}, \methodname{} breaks the over-conservative behavior constraint, enabling the discovery and exploitation of novel behaviors beyond the reference distribution.

\begin{figure}[t]
\centering
\vspace{-0.5cm}
\includegraphics[width=0.9\textwidth]{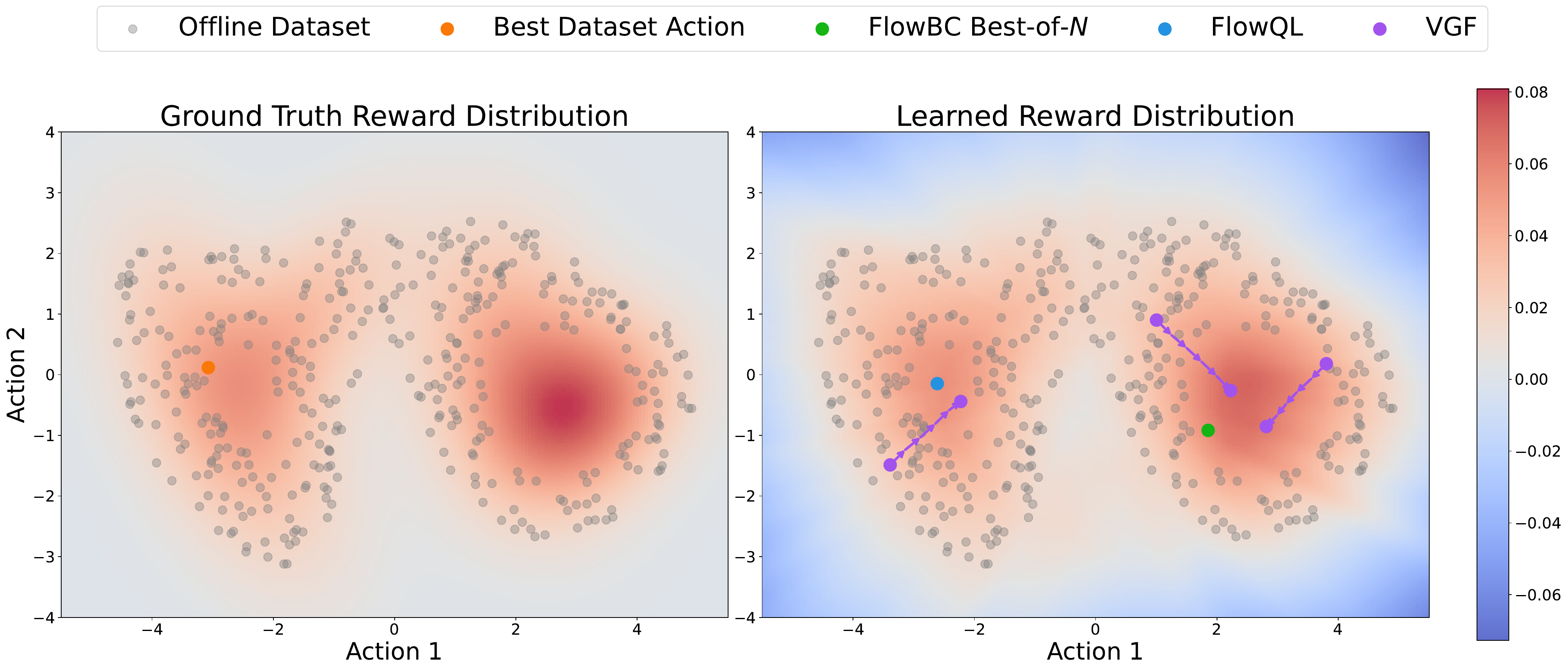}
\vspace{-0.1cm}
\caption{\footnotesize 
\textbf{Toycase results}. \methodname{} generates actions with higher ground-truth reward than other methods. 
}
\label{fig_toycase}
\vspace{-0.2cm}
\end{figure}

% \noindent \textbf{What benefits does VGF bring?} \
\subsection{Discussion} %%AZ.9.24: more descriptive section title?
% The connection to previous work is like BCQ, FQL, Diffusion-DICE and CFGRL (difference with Diffusion-DICE and CFGRL: VGF is doing TD learning with better value functions). 
Below, we outline several distinctive advantages of \methodname{} and discuss the connection with prior work.

\textbf{(1) Stable optimization with reduced conservatism.} \
Compared with methods based on reparameterized policy gradient~\citep{wang2023offline},
\methodname{} avoids optimization instability caused by backpropagating through time. Futhermore, \methodname{} is optimized to find the best reward-maximization policy within a fixed behavior constraint, which is more aligned with \Cref{eq_behavior_regularized_rl}.

\textbf{(2) Implicit policy with multimodal expressivity.} \
While bypassing explicit policy parameterization, the implicit policy in \methodname{} is still expressive enough to capture a multimodal distribution. Owing to the usage of gradient flow, \methodname{} naturally preserves and sharpens multiple high-value modes from the reference distribution instead of collapsing to a single one.
This is different from BCQ~\citep{fujimoto2019off} where a Gaussian residual policy with limited expressivity is learned on top of the reference policy. 
Note that in the offline RL setting, \methodname{} remains versatile to the usage of different advanced generative models, e.g., Diffusion models~\citep{song2021maximum} or Flow models~\citep{lipman2022flow}, to generate samples from the reference distribution given only an offline dataset.
% we have the following matching loss $L_{\text{Flow}}(\theta)$~\citep{lipman2022flow}.
% \begin{equation*}
% \mathbb{E}_{\substack{s \sim \mathcal{D}, a=x^1 \sim \mu, \\ x^0 \sim \mathcal{N}(0, I), \\ t \sim \mathrm{Unif}([0, 1])}} \left[\|v_\theta(t, s, x^t) - (x^1 - x^0)\|^2\right],
% \end{equation*}
% where $v_\theta(t, s, x)$ is a state-and-time dependent vector field with parameter $\theta$.
% For the use of diffusion models, we have the following matching loss $L_{\text{Diffusion}}(\theta)$ based on \citet{ho2020denoising}.
% \begin{equation}
% \label{eq_diffusion_loss}
% \mathbb{E}_{s \sim \mathcal{D}, a \sim \mu, \epsilon \sim \mathcal{N}(0, I), t \sim \mathrm{Unif}([1, T])}  \left[\| \epsilon-\epsilon_{\theta}\left(\sqrt{\bar{\alpha}_{t}} a+\sqrt{1-\bar{\alpha}_{t}} \epsilon, s, t\right)\|^2\right].
% \end{equation}

\textbf{(3) Adaptive scaling during test-time.} \
One intriguing property of \methodname{} is that it enables adaptive test-time scaling via varying the transport budget \textbf{without} any retraining.
For example, when the value function $R$ can generalize well, the performance will scale with the number of test-time flow steps, which could be different from the number of train-time flow steps.  
However, when the value function has large extrapolation errors, by setting the test-time flow steps to 0, \methodname{} reduces to Best-of-$N$ sampling methods~\citep{chen2022offline,hansen2023idql}. One difference in this case is that \methodname{} learns the value function by TD learning (since the train-time flow step is not 0) instead of in-sample learning~\citep{kostrikov2021iql,xu2023offline}. We find in practice that TD learning enables better stitching and generalization.
This difference makes \methodname{} \textbf{fundamentally different} from Diffusion-based methods~\citep{mao2024diffusion,frans2025diffusion} that can also do adaptive generation via adjusting the guidance weight but rely on in-sample value learning.

% \noindent \textbf{Connections to prior work} \quad
\begin{wrapfigure}{R}{0.45\textwidth}
\vspace{-0.85cm}
\begin{minipage}{0.45\textwidth}
\begin{algorithm}[H]
\small
\caption{Value Gradient Flow}
\label{alg_vgf}
\begin{algorithmic}[1]
    \BeginBox[fill=mycolor!10]
    \Function{\textsc{\textcolor{mycolor}{VGF}}}{$s, \hat{\mu}, R, L_{\mathrm{test}}$}
        \State Get $a_{N}^{0} \sim \hat{\mu}(\cdot | s)$
        \For{$l = 0, 1, \dots, L_{\mathrm{test}}-1$}
            \State Get $a_{N}^{l+1}$ using $R$ and $a_{N}^{l}$ by Eq.~(\ref{eq_svgd})
        \EndFor
        \State \Return $a_{N}^{L_{\mathrm{test}}}$
    \EndFunction
    \EndBox

    \Require $\mathcal{D}$, $L_{\mathrm{train}}$, $L_{\mathrm{test}}$, $\epsilon$
    \LComment{\color{mycolor} Value Training (offline RL)}
    \For{$t=1, 2, \cdots, M$}
        \State Sample transitions $(s, a, r, s') \sim \mathcal{D}$
        \State Train behavior cloning policy $\hat{\mu}$
        \State Get $a_{N}^{L_{\mathrm{train}}}
            = \text{\textsc{\textcolor{mycolor}{VGF}}}(s',\hat{\mu},Q,L_{\mathrm{train}})$
        \State Train $Q$ using $a_{N}^{L_{\mathrm{train}}}$ by TD learning
    \EndFor

    \LComment{\color{mycolor} Evaluation (offline RL and RLHF)}
    \State Get initial state $s$, set $d$ as False
    \While{not $d$}
        \State Get $a_{N}^{L_{\mathrm{test}}}
            = \text{\textsc{\textcolor{mycolor}{VGF}}}(s,\hat{\mu},R,L_{\mathrm{test}})$
        \State Get best-of-$N$ $a^*$ from $a_{N}^{L_{\mathrm{test}}}$ by Eq.~(\ref{eq_bon})
        \State Roll out $a^*$ and get $(s', r, d)$
        \State Set $s \gets s'$
    \EndWhile
\end{algorithmic}
\end{algorithm}
\end{minipage}
\vspace{-0.4cm}
\end{wrapfigure}

\noindent \textbf{Practical consideration.} \
We summarize the pseudo-code in \Cref{alg_vgf}.
To reduce sampling cost, we use a small number of VGF particles ($N=5$) across all experiments. At test time, since we need to choose one particle to do evaluation, we use best-of-$N$ sampling from all VGF particles based on the value/reward function. 
In offline RL, given the offline dataset, we train a behavior cloning policy to generate samples from $\mu$.
The $Q$-function is trained using TD-learning, and we average over all particles when computing the target $Q$-values. We use the \textbf{w/o MaxEnt} objective in \Cref{eq_svgd}. To accelerate training and inference, we additionally train a network $f(s,a)$ to capture the gradient of $Q$-function via $\min_f \mathbb{E}_{(s,a) \sim \mathcal{D}}[(f(s,a) - \nabla_a Q(s,a))^2]$.

\noindent \textbf{A toy example.} \
We use a toy example to illustrate the mechanism of \methodname{}.
% that \methodname{} tends to be a better method than FQL and traditional KL regularization in terms of exploring an out-of-support area with high rewards.
We construct a 2-D continuous control bandit task with a bimodal ground-truth reward distribution, where the offline dataset is generated from sampling from sub-optimal reward regions, as demonstrated in Figure~\ref{fig_toycase}. 
We are interested in studying the behavior of the following three different behavior-regularized RL methods.
Note that all three methods fit a learned reward model using $L_2$ loss and a BC flow model using flow-matching loss.
\textbf{FlowQL}~\citep{park2025flow}: This method represents the first group of methods in section 2. FlowQL additionally trains a one-step flow model as the policy, which is used during evaluation. We carefully tune the coefficient $\beta$ to ensure the best performance.
\textbf{FlowBC Best-of-$N$}: This method represents the second group of methods in section 2. In this case, we sample $N=20$ actions from the BC flow model and do best-of-$N$ sampling using the learned reward function.
\textbf{\methodname{}}: For this task, we set particle number $N=3$, $L_{\mathrm{test}}=5$ and $\alpha=0.1$. 
% \haoran{this baseline has a closed form solution which amounts to rejected sampling using r(s,a)}
% \item Baseline 2: Use DQL \haoran{the second baseline is max r(s,a) + bc(pi, mu) where bc is not KL and pi is either diffusion models or flow models, use FQL is enough to showcase, no need for DQL}

We plot the action generated by FlowQL and FlowBC Best-of-$N$, along with the flow trajectories of particles generated by VGF, in Figure~\ref{fig_toycase}. As shown, even if the learned reward model gets some error, the implicit policy in \methodname{} demonstrates successful and effective exploration of the area with high ground-truth reward.  
% Although not all the sampled actions eventually reach the higher-reward area, a certain proportion of the actions do if their initial positions are close to it. 
% This indicates that \methodname{} enables the sampled actions to break the restriction of the dataset support and achieve outstanding performance. 
By contrast, FlowQL is shown to be misled by the error of the learned reward model, generating actions with suboptimal values. 
% We conjecture that this is because the actor loss of FlowQL consists of both one-step distillation loss and reward-maximization loss, and jointly optimizing these losses leads to erroneous generalization when both the behavior is suboptimal and the learned reward function is inaccurate. 
Although best-of-$N$ sampling from the FlowBC model could improve the best dataset action, it still falls within the support of the suboptimal behavior distribution due to over-conservatism, which aligns well with the theoretical analysis.

\section{Related Work}

\noindent \textbf{Offline RL.} \
To address distributional shift, most model-free offline RL methods augment off-policy learning with behavior regularization. 
This can appear explicitly as divergence penalties that constrain the learned policy toward the dataset distribution \citep{wu2019behavior,kumar2019stabilizing,fujimoto2021minimalist} or implicitly via weighted behavior cloning and advantage-weighted updates \citep{wang2020critic,nair2020accelerating}. 
% A complementary line modifies value learning to penalize OOD actions and promote pessimism \citep{kumar2020conservative,kostrikov2021offline}, with related approaches leveraging uncertainty estimates or learned distance functions to discourage extrapolation \citep{an2021uncertainty,li2022data}. 
% While effective at reducing overestimation, many of these methods rely on actor–critic formulations with unimodal Gaussian policies; several studies have noted their difficulty in representing multi-modal action distributions and the resulting conservatism within the data support \citep{wang2022diffusion,hansen2023idql,chen2022offline}. 
To improve policy expressiveness, recent work adopts expressive generative models as policies, including Decision Transformer \citep{chen2021decision}, diffusion-based policies \citep{wang2022diffusion,chen2022offline,hansen2023idql}, flow matching policies \citep{park2025flow} and consistency-model policies \citep{ding2023consistency}. 
% Others perform trajectory-level planning with generative models and execute the action for the current state from synthesized high-return trajectories \citep{janner2022planning,ajay2022conditional}. Despite these advances, relatively little work examines how errors in expressive behavior models can be exploited during generation or how to regularize improvement without resorting to explicit distance penalties.
% \citep{liu2023budgeting} This paper considers Budgeting Counterfactual actions as regularization, while we start from the optimal transport level. 
Compared with these methods, VGF bypasses explicit policy parameterization, improving training stability while retaining strong expressiveness.
There are several recent works that use the $Q$ action gradient to guide the policy learning \citep{yang2023policy,psenka2023learning,mark2024policy,xu2025uni,li2026q}. However, their motivations and derivations are fundamentally different from VGF. Among them, the most closely related methods are PA-RL \citep{mark2024policy} and QAM \citep{li2026q}. Compared to PA-RL, VGF avoids explicitly parameterizing a policy, which can otherwise limit flexibility for adaptive test-time scaling. In contrast to QAM, which optimizes a KL-regularized policy that still remains within the support of the behavior distribution, VGF is encouraged to move beyond the behavior distribution due to its implicit regularization. 
% This improves offline training and promotes more effective online exploration.

\noindent \textbf{Reinforcement Learning from Human Feedback.} \
% In RLHF, policy models can exploit imperfections in reward models, a phenomenon often termed reward over-optimization \citep{gao2023scaling}, also known as reward hacking or reward gaming \citep{amodei2016concrete,skalse2022,pang2023}. 
% Many studies analyze this effect in synthetic setups using strong “gold” models for labeling and evaluation \citep{gao2023scaling,moskovitz2023constraining,coste2023ensemble}. A central mitigation strategy adds a KL penalty to the reward or training objective, encouraging the policy to remain close to a supervised-finetuned reference model \citep{kullback1951,stiennon2020learning,ouyang2022training,bai2022}. 
% Other approaches use ensembles or early-stopping-like constraints to limit over-optimization while controlling KL \citep{coste2023ensemble,moskovitz2023constraining}. 
% However, such explicit penalties introduce a reward–KL trade-off that is sensitive to coefficient tuning and can be overly conservative within the in-distribution region of the reward model. 
In RLHF, policy models can exploit imperfections in learned reward models, a phenomenon often termed reward over-optimization \citep{gao2023scaling} and also discussed as reward hacking or reward gaming \citep{amodei2016concrete,skalse2022,pang2023}. Many studies analyze this effect in synthetic setups that substitute expensive human evaluation with strong "gold" models for labeling and assessment \citep{gao2023scaling,moskovitz2023constraining,coste2023ensemble}. A prevailing mitigation strategy augments the reward or training objective with a KL penalty to a supervised-finetuned reference model \citep{kullback1951,stiennon2020learning,ouyang2022training,bai2022}. Other approaches employ ensembles or early-stopping-style constraints to curb over-optimization while controlling KL \citep{coste2023ensemble,moskovitz2023constraining}. However, explicit penalties inevitably introduce a reward-KL trade-off that is sensitive to coefficient tuning and can be overly conservative towards the reference support.

\noindent \textbf{Optimal transport in RL.} \
Optimal transport (OT) provides a geometry over distributions that has proved useful in multiple RL settings. In distributional RL, Wasserstein metrics underpin return-distribution learning via quantile-regression objectives, improving stability and control~\citep{dabney2018distributional}. OT has also been used to align occupancy distributions for imitation and offline learning. For example, Sinkhorn-based matching or primal Wasserstein formulations that shape rewards and enable cross-domain alignment~\citep{dadashi2020primal}. Beyond matching, an OT viewpoint motivates transporting probability mass toward value-preferred regions, inspiring flow/particle-style policy updates and robust formulations that explicitly constrain distributional shift. Our work follows this trajectory but emphasizes transport from the reference distribution to the optimal policy distribution using value gradients. 
% This yields an implicit policy and turns the transport budget itself into behavior regularization without tightly coupled policy-optimization loops~\citep{peyre2019computational}. 
Other work like PPL~\citep{asadulaev2024rethinking} considers transport between states and partial action distributions, whereas \methodname{} operates in the action space.
% Perhaps the most similar paper to us is xxx, which considers policy optimization as wasserstein gradient flow. However, the particle-based solution is derived to be different where VGF uses insight from behavior-regularized RL, and they study the online RL setting while our method focuses on the behavior-regularized RL setting.

% D4RL
\begin{table}[t]
\centering
% \small
% \begin{threeparttable}
\caption{
\footnotesize 
\textbf{D4RL offline RL results}. 
Scores are averaged over the final 10 evaluations across 5 seeds with standard deviation reported, we highlight the best score in integer-level. \methodname{} demonstrates superior performance on most tasks, especially those challenging ones.
% \haoran{change idql to sfbc}
% UNIVR achieves the highest scores in most of tasks.
}
% Our method outperforms prior methods on the challenging Ant Maze tasks, which require dynamic programming, and is competitive with the best prior methods on the locomotion tasks.  %
\label{tab_d4rl_result}
% \scriptsize
\resizebox{1.0\linewidth}{!}{
% \begin{tabular}{l||ccccccccc|cc}
% \toprule
% Dataset                      & BC    & 10$\%$BC & BCQ   & DT    & One-step & TD3+BC & CQL   & IQL                  & IVR  & \textcolor{mycolor}{\textbf{\makecell{UNIVR \\ (Gaussian)}}} & \textcolor{mycolor}{\textbf{\makecell{UNIVR \\ (Diffusion)}}}  \\
\begin{tabular}{lcccccc|c} 
\toprule
\multicolumn{1}{c}{} & \multicolumn{3}{c}{\texttt{Gaussian Policy}} & \multicolumn{3}{c}{\texttt{Diffusion/Flow Policy}} & \multicolumn{1}{c}{\texttt{w/o Policy}} \\
\cmidrule(lr){2-4} \cmidrule(lr){5-7} \cmidrule(lr){8-8}
\texttt{Dataset}  &  \texttt{TD3+BC}   & \texttt{IQL} & \texttt{IVR} & \texttt{Diffusion-QL} & \texttt{SfBC} & \texttt{FQL} & \texttt{VGF (ours)} \\ 
\hline
\texttt{halfcheetah-m}   & $48.3$ & $47.4$ {\tiny $\pm 0.2$} & $48.3$ {\tiny $\pm 0.2$} & $51.1$ {\tiny $\pm 0.5$} & $45.9$ {\tiny $\pm 2.2$} & $55.6$ {\tiny $\pm 0.2$} & \colorbox{light_mycolor}{$57.1$} {\tiny $\pm 0.1$} \\
\texttt{hopper-m}        & $59.3$ & $66.3$ {\tiny $\pm 5.7$} & $75.5$ {\tiny $\pm 3.4$} & $90.5$ {\tiny $\pm 4.6$} & $57.1$ {\tiny $\pm 4.1$} & $60.6$ {\tiny $\pm 0.1$} & \colorbox{light_mycolor}{$97.9$} {\tiny $\pm 2.0$} \\
\texttt{walker2d-m}      & $83.7$ & $72.5$ {\tiny $\pm 8.7$} & $84.2$ {\tiny $\pm 4.6$} & $87.0$ {\tiny $\pm 0.9$} & $77.9$ {\tiny $\pm 2.5$} & $65.9$ {\tiny $\pm 0.3$} & \colorbox{light_mycolor}{$89.4$} {\tiny $\pm 0.8$} \\
\texttt{halfcheetah-m-r} & $44.6$ & $44.2$ {\tiny $\pm 1.2$} & $44.8$ {\tiny $\pm 0.7$} & $47.8$ {\tiny $\pm 0.3$} & $37.1$ {\tiny $\pm 1.7$} & $48.3$ {\tiny $\pm 0.3$} & \colorbox{light_mycolor}{$49.1$} {\tiny $\pm 0.1$} \\
\texttt{hopper-m-r}      & $60.9$ & $95.2$ {\tiny $\pm 8.6$} & $99.7$ {\tiny $\pm 3.3$} & \colorbox{light_mycolor}{$101.3$} {\tiny $\pm 0.6$} & $86.2$ {\tiny $\pm 9.1$} & $50.7$ {\tiny $\pm 2.7$} & $99.0$ {\tiny $\pm 1.1$} \\
\texttt{walker2d-m-r}    & $81.8$ & $76.1$ {\tiny $\pm 7.3$} & $81.2$ {\tiny $\pm 3.8$} & $95.5$ {\tiny $\pm 1.5$} & $65.1$ {\tiny $\pm 5.6$} & $38.8$ {\tiny $\pm 1.1$} & \colorbox{light_mycolor}{$97.8$} {\tiny $\pm 1.6$} \\
\texttt{halfcheetah-m-e} & $90.7$ & $86.7$ {\tiny $\pm 5.3$} & $94.0$ {\tiny $\pm 0.4$} & $96.8$ {\tiny $\pm 0.3$} & $92.6$ {\tiny $\pm 0.5$} & \colorbox{light_mycolor}{$102.1$} {\tiny $\pm 0.6$} & $99.1$ {\tiny $\pm 0.3$} \\
\texttt{hopper-m-e}      & $98.0$ & $101.5$ {\tiny $\pm 7.3$} & \colorbox{light_mycolor}{$111.8$} {\tiny $\pm 2.2$} & $111.1$ {\tiny $\pm 1.3$} & $108.6$ {\tiny $\pm 2.1$} & $76.7$ {\tiny $\pm 0.6$} & $98.3$ {\tiny $\pm 3.3$} \\
\texttt{walker2d-m-e}    & $110.1$ & \colorbox{light_mycolor}{$110.6$} {\tiny $\pm 1.0$} & $110.0$ {\tiny $\pm 0.8$} & $110.1$ {\tiny $\pm 0.3$} & $109.8$ {\tiny $\pm 0.2$} & $102.6$ {\tiny $\pm 0.2$} & \colorbox{light_mycolor}{$110.5$} {\tiny $\pm 1.5$} \\ 
\hline
\texttt{antmaze-u}       & $78.6$ & $85.5$ {\tiny $\pm 1.9$} & $92.2$ {\tiny $\pm 1.4$} & $93.4$ {\tiny $\pm 3.4$} & $92.0$ {\tiny $\pm 2.1$} & $96$ {\tiny $\pm 1.6$} & \colorbox{light_mycolor}{$98.0$} {\tiny $\pm 1.8$} \\
\texttt{antmaze-u-d}     & $71.4$ & $66.7$ {\tiny $\pm 4.0$} & $74.0$ {\tiny $\pm 2.3$} & $66.2$ {\tiny $\pm 8.6$} & $85.3$ {\tiny $\pm 3.6$} & $89$ {\tiny $\pm 2.3$} & \colorbox{light_mycolor}{$94.3$} {\tiny $\pm 1.4$} \\
\texttt{antmaze-m-p}     & $10.6$ & $72.2$ {\tiny $\pm 5.3$} & $80.2$ {\tiny $\pm 3.7$} & $76.6$ {\tiny $\pm 10.8$} & $81.3$ {\tiny $\pm 2.6$} & $78.0$ {\tiny $\pm 2.6$} & \colorbox{light_mycolor}{$89.4$} {\tiny $\pm 3.1$} \\
\texttt{antmaze-m-d}     & $3.0$ & $71.0$ {\tiny $\pm 3.2$} & $79.1$ {\tiny $\pm 4.2$} & $78.6$ {\tiny $\pm 10.3$} & $82.0$ {\tiny $\pm 3.1$} & $71.0$ {\tiny $\pm 3.4$} & \colorbox{light_mycolor}{$86.7$} {\tiny $\pm 2.8$} \\
\texttt{antmaze-l-p}     & $0.2$ & $39.6$ {\tiny $\pm 4.5$} & $53.2$ {\tiny $\pm 4.8$} & $46.4$ {\tiny $\pm 8.3$} & $59.3$ {\tiny $\pm 14.3$} & \colorbox{light_mycolor}{$84.0$} {\tiny $\pm 2.9$} & {$82.5$} {\tiny $\pm 3.6$} \\
\texttt{antmaze-l-d}     & $0.0$ & $47.5$ {\tiny $\pm 4.4$} & $52.3$ {\tiny $\pm 5.2$} & $56.6$ {\tiny $\pm 7.6$} & $45.5$ {\tiny $\pm 6.6$} & $83.0$ {\tiny $\pm 3.8$} & \colorbox{light_mycolor}{$83.8$} {\tiny $\pm 4.5$} \\
\bottomrule
\end{tabular}
}
% \vspace{-0.3cm}
\end{table}

\section{Experiments}
% \haoran{not done yet}
% The goal of our experiments is to evaluate the efficacy of \methodname{} in improving offline RL and RLHF. To this end, we compare \methodname{} to state-of-the-art methods, and answer the following questions: \textbf{(1)} Does \methodname{} improve performance when compared to using similar-sized networks on benchmark tasks? and 
% \textbf{(2)} How does the use of iterative computation via \methodname{} compare with the use of ``parallel'' computation of a neural network ensemble and ``sequential'' iterative computation driven by ResNets of comparable size? We then run several experiments to understand the behavior of \methodname{} critics. Furthermore, we run a variety of ablation studies to understand the design choices that drive the efficient use of iterative computation, including the roles of \textbf{a)} tuning the width of initial noise sample, \textbf{b)} categorical representations of the interpolant input, and \textbf{c)} Fourier-basis time embeddings. 
The goal of our experiments is to evaluate the efficacy of \methodname{} in improving offline RL and RLHF. We evaluate the performance of \methodname{} on D4RL and OGBench and compare it with prior methods. We also provide an ablation study on important hyperparameters and investigate the adaptive scaling behavior during test time in \methodname{} to gain a deeper understanding of its mechanism.

\subsection{Offline RL Results}
% \noindent \textbf{Offline RL tasks and datasets.} \
% In the offline setting, we evaluate VGF on the D4RL benchmark~\citep{fu2020d4rl} and compare it with several state-of-the-art algorithms. For the evaluation tasks, we select MuJoCo locomotion tasks and AntMaze navigation tasks which require both locomotion and navigation. While MuJoCo tasks are popular in offline RL, AntMaze tasks are more challenging due to their stronger need for selecting optimal parts of different trajectories to perform stitching. For baseline algorithms, we selected state-of-the-art methods not only from traditional methods that use a Gaussian policy but also methods that use diffusion models. 

\textbf{D4RL Benchmark Datasets.} \ We evaluate the performance of VGF on the D4RL benchmark~\citep{fu2020d4rl}, and compare it with several algorithms based on Gaussian policy, diffusion policy and flow policy. Gaussian-policy-based baselines include TD3+BC~\citep{fujimoto2021minimalist}, IQL~\citep{kostrikov2021offline}, and IVR~\citep{xu2023offline}. We also select Diffusion-QL~\citep{wang2022diffusion} and SfBC~\citep{chen2022offline} as diffusion-policy-based baselines, and FQL~\citep{park2025flow} as a flow-policy-based baseline. The evaluation tasks include MuJoCo, a set of locomotion tasks, and AntMaze, a series of navigation tasks. The results in Table~\ref{tab_d4rl_result} show that VGF outperforms all of the baselines in most datasets. Note that FQL has poor performance on some Mujoco datasets even under carefully hyperparameter tuning. We would like to highlight that VGF achieves much higher scores than prior methods on challenging AntMaze datasets.

\textbf{OGBench Datasets.} \ Another benchmark that we use to evaluate is OGBench~\citep{ogbench_park2025}, which provides a variety of goal-conditioned tasks and datasets across robotic locomotion and manipulation. Among the reward-based single-task settings, we select 4 locomotion and 5 manipulation environments, each of which consists of 5 tasks. We take the average score of the 5 tasks to be the indicator of performance in these environments. We compare our results with several of the state-of-the-art algorithms, including ReBRAC~\citep{tarasov2023revisiting}, Flow BRAC~\citep{wu2019behavior}, IDQL~\citep{hansen2023idql} and FQL~\citep{park2025flow}. ReBRAC leverages a monolithic Q network and a Gaussian policy and achieves competitive performance on offline RL datasets. Flow BRAC replaces Gaussian policy with flow policy in behavior-regularized actor-critic algorithms. IDQL trains a diffusion BC model along with a value function learned by IQL to do best-of-$N$ sampling. FQL is a recently proposed method that introduces a one-step flow policy that distills from the BC flow policy to avoid unstable gradient backpropagation through time. 
Table~\ref{tab_ogbench_result} summarizes the results of VGF and its comparison with baselines on OGBench. VGF achieves better performance than prior methods in most of the environments, especially those hard ones where FQL attains performance below 50\% success rate (\texttt{cube-double}, \texttt{puzzle-3x3} and \texttt{puzzle-4x4}).

\textbf{Online Finetuning.} \ We evaluate the efficacy of VGF in the online RL fine-tuning setting. Here, we first train agents offline for 1M steps and subsequently fine-tune them online for an additional 1M steps. Across all challenging tasks, we find that VGF provides a stronger initialization compared to FQL and also adapts more rapidly during online interaction and converges to higher final performance.

\begin{table}[t]
% \vspace{-10pt}
\caption{
\footnotesize 
\textbf{OGBench offline RL results}. 
Scores are averaged over the final 10 evaluations across 5 seeds with standard deviation reported, we highlight the best score in integer-level. \methodname{} achieves competitive or superior performance compared to prior approaches, especially on hard tasks.}
\label{tab_ogbench_result}
\centering
\vspace{-0.1cm}
\resizebox{0.95\linewidth}{!}
{
\begin{threeparttable}
\begin{tabular}{lcccccc|c}
\toprule
\multicolumn{1}{c}{} &
\multicolumn{3}{c}{\texttt{Gaussian Policy}} &
\multicolumn{3}{c}{\texttt{Diffusion/Flow Policy}} &
\multicolumn{1}{c}{\texttt{w/o Policy}} \\
\cmidrule(lr){2-4} \cmidrule(lr){5-7} \cmidrule(lr){8-8}
\texttt{Dataset (5 tasks)} &
\texttt{BC} & \texttt{IQL} & \texttt{ReBRAC} &
\texttt{FBRAC} & \texttt{IDQL} & \texttt{FQL} &
\texttt{VGF (ours)} \\
\midrule

% \texttt{antmaze-large}  &
% $11 \pm 1$ & $53 \pm 3$ & \colorbox{light_mycolor}{$81 \pm 5$} &
% $60 \pm 6$ & $21 \pm 5$ & $79 \pm 3$ & $77 \pm 5$ \\

\texttt{antmaze-giant}  &
$0 \pm 0$ & $4 \pm 1$ & \colorbox{light_mycolor}{$26 \pm 8$} &
$4 \pm 4$ & $0 \pm 0$ & $9 \pm 6$ & $3 \pm 1$ \\

\texttt{humanoidmaze-medium}  &
$2 \pm 1$ & $33 \pm 2$ & $22 \pm 8$ &
$38 \pm 5$ & $1 \pm 0$ & $58 \pm 5$ & \colorbox{light_mycolor}{$72 \pm 1$} \\

\texttt{humanoidmaze-large}  &
$1 \pm 0$ & $2 \pm 1$ & $2 \pm 1$ &
$2 \pm 0$ & $1 \pm 0$ & $4 \pm 2$ & \colorbox{light_mycolor}{$15 \pm 2$} \\

\texttt{antsoccer-arena}  &
$1 \pm 0$ & $8 \pm 2$ & $0 \pm 0$ &
$16 \pm 1$ & $12 \pm 4$ & $60 \pm 2$ & \colorbox{light_mycolor}{$63 \pm 4$} \\

\texttt{cube-single}  &
$5 \pm 1$ & $83 \pm 3$ & $91 \pm 2$ &
$79 \pm 7$ & $95 \pm 2$ & \colorbox{light_mycolor}{$96 \pm 1$} & \colorbox{light_mycolor}{$96 \pm 1$} \\

\texttt{cube-double}  &
$2 \pm 1$ & $7 \pm 1$ & $12 \pm 1$ &
$15 \pm 3$ & $15 \pm 6$ & $29 \pm 2$ & \colorbox{light_mycolor}{$70 \pm 8$} \\

\texttt{scene} &
$5 \pm 1$ & $28 \pm 1$ & $41 \pm 3$ &
$45 \pm 5$ & $46 \pm 3$ & $56 \pm 2$ & \colorbox{light_mycolor}{$60 \pm 1$} \\

\texttt{puzzle-3x3}  &
$2 \pm 0$ & $9 \pm 1$ & $21 \pm 1$ &
$14 \pm 4$ & $10 \pm 2$ & $30 \pm 1$ & \colorbox{light_mycolor}{$75 \pm 4$} \\

\texttt{puzzle-4x4} &
$0 \pm 0$ & $7 \pm 1$ & $14 \pm 1$ &
$13 \pm 1$ & $29 \pm 3$ & $17 \pm 2$ & \colorbox{light_mycolor}{$45 \pm 4$} \\

\bottomrule
\end{tabular}
\end{threeparttable}
}
\vspace{-0.1cm}
\end{table}

\begin{figure}[t]
\centering
% \vspace{-0.5cm}
\includegraphics[width=1.0\textwidth]{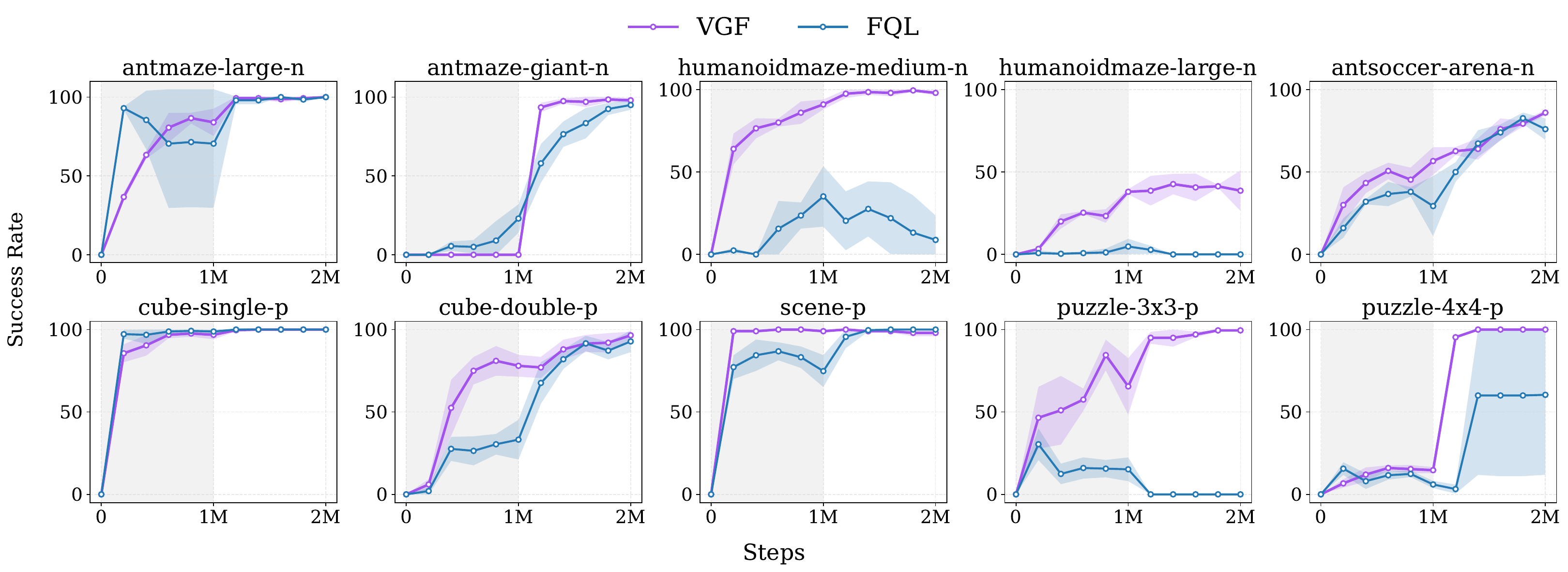}
% \vspace{-0.1cm}
\caption{\footnotesize 
\textbf{OGBench offline-to-online RL results.} Learning curves for online fine-tuning of VGF and FQL across all default tasks. VGF not only provides a stronger initialization from offline training but also leads to faster adaptation and higher final success rates. The shaded gray area denotes offline training.
}
\label{fig:vgf_off2on}
% \vspace{-0.2cm}
\end{figure}

\subsection{RLHF Results}

% \noindent \textbf{RLHF tasks and datasets.} 
\begin{wraptable}{r}{0.5\textwidth}
\vspace{-0.4cm} % tighten top gap (optional)
\centering
\footnotesize
\begin{tabular}{lcc}
\toprule
{} & \texttt{TL;DR} & \texttt{Anthropic-HH} \\
\cmidrule(lr){2-2} \cmidrule(lr){3-3}
\texttt{Model} & WR\% (vs ref) & WR\% (vs chosen)  \\
\midrule
\texttt{Pythia-SFT} & $48.5$ & $46.2$  \\
\midrule
\texttt{PPO}      & $57.3$ & $45.5$  \\
\texttt{DPO}      & $61.2 $ & $51.5 $  \\
\texttt{Best-of-$N$}  & $58.3 $ & $49.0 $  \\
\texttt{VGF (Ours)} & \cellcolor{light_mycolor}$68.1$ &
\cellcolor{light_mycolor}$59.0  $  \\
\bottomrule
\end{tabular}
\caption{\textbf{RLHF results}. \methodname{} outperforms baseline RLHF methods by having higher win-rates on TL;DR and Anthropic-HH dataset.}
\label{tab:rlhf}
\vspace{-6pt} % tighten bottom gap (optional)
\end{wraptable}
We report results on the TL;DR Summarize~\citep{stiennon2020learning} and Anthropic Helpful and Harmless Dialogue~\citep{bai2022} datasets. 
The training split of \texttt{TL;DR} dataset contains 116k human-written instructions and 93k human-annotated preference pairs. The pre-processed \texttt{Anthropic-HH} dataset contains 112k training preference pairs.
Both the reward model and the SFT policy are initialized from the same Pythia-2.8B base model, trained on either human demonstrations (\texttt{TL;DR}) or the chosen responses (\texttt{Anthropic-HH}). For evaluation metrics, we calculate win rates (WR\%) judged by GPT-4 comparing outputs from the initialization and aligned models. As shown in Table~\ref{tab:rlhf}, \methodname{} outperforms all baseline RLHF methods by a large margin.

\subsection{Understanding the Properties and Behavior of \methodname{}}
% To better understand the benefits of iterative compute in \methodname, we analyze its scaling behavior and compare it to different scaling approaches. Specifically, we study: \textbf{(i)} how the number of flow-integration steps controls the expressivity of \methodname; \textbf{(ii)} how \methodname’s iterative computation compares to monolithic critic scaling; and \textbf{(iii)} the importance of applying supervision to the velocity field at every flow step.
% We present our ablations and findings via the following Q\&As.

To better understand the mechanism and behavior of VGF, we investigate the importance of hyperparameter selection and the test-time scaling property by conducting some ablation study. 
% We present our analysis and findings via a Q\&A style.

\noindent \textbf{What are the important hyperparameters of VGF?} \
There are three hyperparameters in VGF: train-time flow steps $L_{\mathrm{train}}$, step size $\epsilon$ and particle number $N$, with $L_{\mathrm{train}}$ being the most important one. 
Train-time flow steps $L_{\mathrm{train}}$ is the number of flow steps we adopt during training. $L$ is directly related to the distance between the reference policy and the learned policy. Intuitively, a larger $L$ means deviating more from the reference policy. In \Cref{fig:vgf_step}, we show that optimal $L_{\mathrm{train}}$ needs to be tuned per task to achieve the best performance.
% Temperature $\alpha$ in \Cref{eq_svgd} serves as the trade-off between reward maximization and dispersing particles. Specifically, a smaller $\alpha$ drives the actions towards the area with higher reward, but with a risk of collapsing to a single mode. The value of $\alpha$ is tuned for each task to seek a balance between reward maximization and dispersion.
% Particle number $N$ affects the multi-modality of the implicit policy in \methodname{}. We find that choosing a small number is enough to represent an expressive distribution. We choose $N=5$ in our experiments and find that setting it to 10 or more has no big difference.
% Below we provide a brief explanation for their importance and provide an ablation study on the effect of $L_{\mathrm{train}}$ in \Cref{fig:vgf_step}. 

\noindent \textbf{Can VGF enable adaptive test-time scaling behavior?} \
The answer is yes by controlling the test-time flow steps $L_{\mathrm{test}}$. 
However, the optimal $L_{\mathrm{test}}$ depends on the specific dataset. 
% As displayed in Figure~\ref{fig:test_time_scaling}, scaling up the number of test-time steps turns out to harm performance in some datasets (e.g., halfcheetah-m-e), while in other datasets (e.g., walker2d-m-r and antmaze-m-d), it can be beneficial. We conjecture that 
In general, scaling up the test-time flow steps is helpful when the value function generalizes well to OOD regions and the offline dataset is of low quality, which requires improving over the reference policy to achieve the best performance.
Note that even when the value function has large extrapolation error, by setting the test-time flow steps to 0, VGF reduces to the best-of-N sampling method, but it can still outperform the reference policy, enabling in-distribution generalization owing to the training of a value function via TD learning.
%%AZ.9.24: but this doesn't explain why more steps harms performance in other settings?

\begin{figure}[t]
\centering
% \vspace{-0.5cm}
\includegraphics[width=0.95\textwidth]{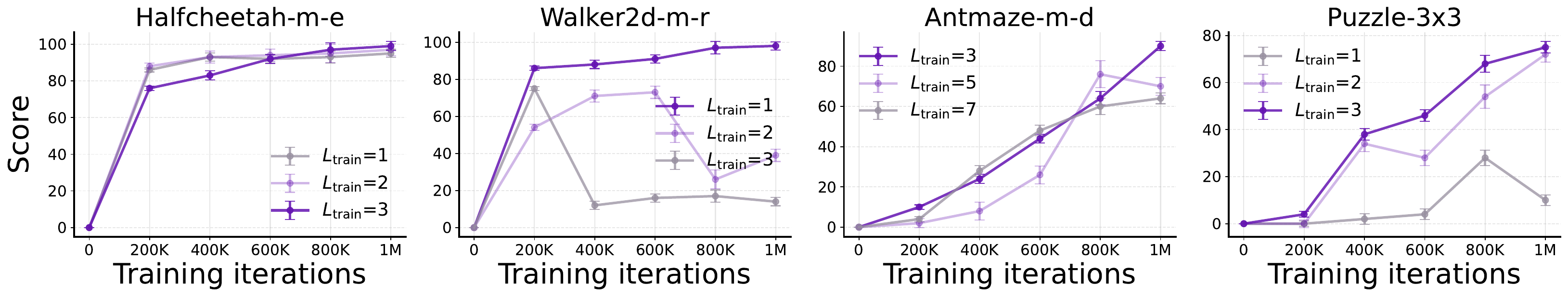}
% \vspace{-0.1cm}
\caption{\footnotesize 
Ablation study on \methodname{} train-time flow steps $L_{\textrm{train}}$.
}
\label{fig:vgf_step}
% \vspace{-0.2cm}
\end{figure}

\begin{figure}[t]
\centering
% \vspace{-0.5cm}
\includegraphics[width=0.95\textwidth]{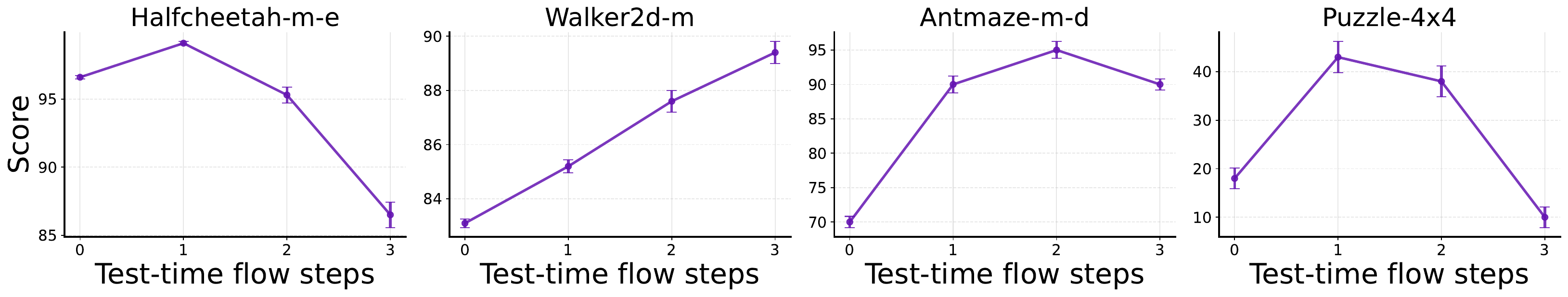}
% \vspace{-0.1cm}
\caption{\footnotesize 
VGF enables adaptive test-time scaling behavior by adjusting test-time flow steps $L_{\textrm{test}}$.
}
\label{fig:test_time_scaling}
% \vspace{-0.2cm}
\end{figure}

% \begin{wrapfigure}{r}{0.6\textwidth}
%     \centering
%     \includegraphics[width=0.6\textwidth]{figures/vgf_steps.pdf}
%     \caption{Ablation study on the effect of VGF iteration step number $L$.}
%     \label{fig:vgf_step}
% \end{wrapfigure}

% \begin{wrapfigure}{r}{0.6\textwidth}
%     \vspace{-3pt}
%     \centering
%     \includegraphics[width=0.6\textwidth]{figures/test_time_scaling.pdf}
%     \caption{Test-time scaling of VGF on different tasks.}
%     \label{fig:test_time_scaling}
% \end{wrapfigure}

\section{Limitations and Future Work}
% Better (Long-horizon) exploration via state occupancy measure or intrinsic rewards
% Better exploration using generative models.
% Better value learning as the KL divergence is still too conservative
In this paper, we propose \methodname{}, a new scalable paradigm that casts behavior-regularized RL as optimal transport from the reference distribution to the optimal policy distribution induced by the value function, where the transport budget serves as the implicit behavior regularization.
\methodname{} uses particle-based gradient flow as the practical solution, yielding an implicit policy with multimodal expressivity while bypassing the need of policy reparameterization. 
\methodname{} is easy to implement, enabling adaptive test-time scaling and achieves strong empirical results on both offline RL and RLHF tasks.
We believe that \methodname{} represents a concrete step toward building unified and scalable behavior-regularized RL algorithms.
One limitation of \methodname{} lies in its ability to handle scenarios where the reference distribution is heavily skewed toward suboptimal behavior. Using technique like distribution reweighting~\citep{xu2025optimal} to enhance performance is one future work.
Another promising direction is to integrate VGF with methods that further improve the expressiveness of the value function~\citep{agrawalla2025floq,dong2025value,dong2026tql}, with the goal of boosting performance on long-horizon tasks.

% \haoran{whether this works in the online or offline-to-online setting? How to train $\mu$ in the online setting is crucial.}

% \section{Reproducibility Statement}
% To ensure the reproducibility of this paper, we detail the theoretical and empirical parts of our results in the main paper and the appendix. In Section \ref{sec:vgf}, we introduce the basic notations used in the theoretical analysis and establish theories to support our claim. Furthermore, in Appendix \ref{sec:omitted_proof}, we provide the detailed proofs of the theoretical results. For empirical details, we briefly introduce the setup of our toy case in Section \ref{sec:vgf}. In Appendix \ref{sec_exp_details}, we elaborate on the benchmark environments, the network architecture and  hyperparameters in our experiments, and provide a more detailed version of empirical results on OGBench. In Appendix \ref{sec_code}, we provide a simple implementation of our method. We will release the code after acceptance. 

\section*{Acknowledgement}
% This work is supported by National Key Research and Development Program of China under Grant (2022YFB2502904). 
We thank members from MIDI lab for discussions on the method and feedback on the early draft of the paper.
This work is partially supported by NSF 2340651, NSF 2402650, DARPA HR00112490431, ARO W911NF-24-1-0193, U. S. Army Research Laboratory and the U. S. Army Research Office under Grant W911NF2010219 and ONR N00014-26-1-2055. HX is supported by the Amazon Fellowship. This research used the computational cluster resource provided by the Texas Advanced Computing Center at UT Austin and Jetstream2 at Indiana University through allocation CIS250850 from the Advanced Cyberinfrastructure Coordination Ecosystem: Services \& Support (ACCESS) program, which is supported by National Science Foundation grants \#2138259, \#2138286, \#2138307, \#2137603, and \#2138296.

{\footnotesize
\bibliographystyle{iclr2026_conference}
\bibliography{iclr2026_conference}
}

%%%%%%%%%%%%%%%%%%%%%%%%%%%%%%%%%%%%%%%%%%%%%%%%%%%%%%%%%%%%

\appendix

\newpage
\newtheorem*{thm1}{\textup{\textbf{\Cref{thm:mmd_bound}}}}
\newtheorem*{thm2}{\textup{\textbf{\Cref{thm:support_shift}}}}
\section{Proof}\label{sec:omitted_proof}

\subsection{Proof of Theorem~\ref{thm:mmd_bound}}
\begin{thm1}
    Assume the value function $R(s,a)$ is $c$-Lipschitz w.r.t the input action $a$. Define the implicit policy that performs \Cref{eq_svgd} for $L$ steps with $N$ particles as $\pi_{N}^{L}$. We have
    \begin{equation*}
        \mathrm{MMD}^2(\mu, \pi_{N}^{L}) = \mathrm{MMD}^2(\pi_{N}^{0}, \pi_{N}^{L}) \leq \frac{2\epsilon L}{\sigma\sqrt{e}}\left(\frac{c}{\alpha} + \frac{1}{\sigma\sqrt{e}}\right).
    \end{equation*}
\end{thm1}
\begin{proof}
    When $R(s,a)$ is $c$-Lipschitz, we know that $\|\nabla_{a_j} R(s,a_j)\|_\infty \leq c$. Additionally, we have another key observation that the kernel $k$ also has Lipschitz property. Formally,
    \begin{align*}
        k(x+\delta,y) - k(x,y) = \exp(-\|x-y\|^2/2\sigma^2) - \exp(-\|x + \delta - y\|^2/2\sigma^2) \leq \frac{\|\delta\|_\infty}{\sigma\sqrt{e}}.
    \end{align*}
    In other words, the kernel $k$ is $\frac{1}{\sigma\sqrt{e}}$-Lipschitz. Combining these two properties, we have
    \begin{align*}
        \|h(s,x)\|_\infty &= \|\mathbb{E}_{a_j \sim \mu_0} \left[ k(a_j, x) \nabla_{a_j} R(s, a_j) / \alpha + \nabla_{a_j} k(a_j, x) \right]\|_\infty \\
        &\leq \sup_{a_j} k(a_j,x) \cdot \frac{c}{\alpha} + \sup_{a_j} \|\nabla_{a_j} k(a_j, x)\|_\infty \\
        &\leq \frac{c}{\alpha} + \frac{1}{\sigma\sqrt{e}}.
    \end{align*}
    The square of MMD distance is calculated as follows:
    \begin{align*}
        \text{MMD}^2(\pi_N^L,\pi_N^0) = \mathbb{E}_{x,y \sim \pi_N^0} k(x,y) + \mathbb{E}_{x,y \sim \pi_N^L} k(x,y) - 2\mathbb{E}_{x \sim \pi_N^0, y \sim \pi_N^L} k(x,y).
    \end{align*}
    Define $K:=\frac{1}{\sigma\sqrt{e}}$ and $H:=\frac{c}{\alpha} + \frac{1}{\sigma\sqrt{e}}$. Let us consider the condition of a single pair of particles $x, y$ after one iteration. We have
    \begin{align*}
        &\quad k(x, y) + k(x+\epsilon h(x), y+\epsilon h(y)) - k(x+\epsilon h(x), y) - k(x, y+\epsilon h(y)) \\
        &\leq |k(x, y) - k(x+\epsilon h(x), y)| + |k(x+\epsilon h(x), y+\epsilon h(y)) - k(x, y+\epsilon h(y))| \\
        &\leq 2\epsilon KH.
    \end{align*}
    Denote $x^{(k)}$ as the particle $x$ after the $k$-th iteration, we have
    \begin{align*}
        \text{MMD}^2(\pi_N^L,\pi_N^0) &= \mathbb{E}_{x,y \sim \pi_N^0} k(x,y) + \mathbb{E}_{x,y \sim \pi_N^L} k(x,y) - 2\mathbb{E}_{x \sim \pi_N^0, y \sim \pi_N^L} k(x,y) \\
        &= \frac{1}{n^2} \sum_{x,y} \left[k(x^{(0)},y^{(0)}) + k(x^{(L)},y^{(L)}) - k(x^{(L)},y^{(0)}) - k(x^{(0)},y^{(L)})\right] \\
        &\leq 2\epsilon KHL.
    \end{align*}
\end{proof}

% \begin{proposition}
%     If the value function $Q(s,a)$ is $c$-Lipschitz with regard to the action $a$, after $L$ iterations of SVGD updates, the target distribution $\mu^L$ has a bounded distance from the behavioral distribution $\mu^0$ of $2\epsilon KHL$, where $\epsilon$ is the update step size, $K:=\frac{1}{\sigma\sqrt{e}}$, and $H:=\frac{c}{\alpha} + \frac{1}{\sigma\sqrt{e}}$.
% \end{proposition}

% \begin{proof}
%     The square of MMD distance is calculated as follows:
%     \begin{align*}
%         \text{MMD}^2(\mu^L,\mu^0) = \mathbb{E}_{x,y \sim \mu^0} k(x,y) + \mathbb{E}_{x,y \sim \mu^L} k(x,y) - 2\mathbb{E}_{x \sim \mu^0, y \sim \mu^L} k(x,y).
%     \end{align*}
%     First, let us consider the condition of a single pair of particles $x, y$ after one iteration. We have
%     \begin{align*}
%         &\quad k(x, y) + k(x+\epsilon h(x), y+\epsilon h(y)) - k(x+\epsilon h(x), y) - k(x, y+\epsilon h(y)) \\
%         &\leq |k(x, y) - k(x+\epsilon h(x), y)| + |k(x+\epsilon h(x), y+\epsilon h(y)) - k(x, y+\epsilon h(y))| \\
%         &\leq 2\epsilon KH.
%     \end{align*}
%     Denote $x^{(k)}$ as the particle $x$ after the $k$-th iteration, we have
%     \begin{align*}
%         \text{MMD}^2(\mu^L,\mu^0) &= \mathbb{E}_{x,y \sim \mu^0} k(x,y) + \mathbb{E}_{x,y \sim \mu^L} k(x,y) - 2\mathbb{E}_{x \sim \mu^0, y \sim \mu^L} k(x,y) \\
%         &= \frac{1}{n^2} \sum_{x,y} \left[k(x^{(0)},y^{(0)}) + k(x^{(L)},y^{(L)}) - k(x^{(L)},y^{(0)}) - k(x^{(0)},y^{(L)})\right] \\
%         &\leq 2\epsilon KHL.
%     \end{align*}
% \end{proof}

\subsection{Proof of Theorem~\ref{thm:support_shift}}
\begin{thm2}
Define the $\epsilon$-support of a distribution $P$ as $\mathrm{supp}_\epsilon(P) := \{x:p(x) \geq \epsilon\}$. We have
\begin{equation*}
    \mathrm{supp}_\epsilon(\pi_{N}^{L}(\cdot|s)) \not\subseteq \mathrm{supp}_\epsilon(\mu(\cdot|s)).
\end{equation*}
\end{thm2}
The proof can be divided into two settings, where the policy distribution is discrete and continuous. We first present the simple proof for the discrete setting.

\begin{proof}
    Denote the support of the behavioral policy as $\text{supp}(\pi_N^0)=\{a_1,...,a_N\}$. If the support of the learned policy from one-step SVGD $\text{supp}(\pi_N^1)$ is a subset of $\text{supp}(\pi_N^0)$, then we know that $a_1$ is updated to $a_i \in \text{supp}(\pi_N^0)$. This indicates that
    \[
    a_i = a_1 + \epsilon h(a_1) = a_1 + \epsilon \mathbb{E}_{a_j \sim \pi_N^0} \left[ k(a_j, a_1) \nabla_{a_j} R(s, a_j) / \alpha + \nabla_{a_j} k(a_j, a_1) \right].
    \]
    Note that a little disturbance in any of the dimensions of $\nabla_{a_j} R(s, a_j)$ will make the equation invalid when the other derivatives of $R$ are fixed. Therefore, the equation is almost surely invalid.
\end{proof}

% \begin{definition}
%     For a continuous distribution $P$, we define its $\epsilon$-support $S$ as the subset of the action space where the probability density function $p$ has a value no less than $\epsilon$ for any element in $S$, i.e., $\forall a \in S$, $p(a) \geq \epsilon$.
% \end{definition}

% We need an assumption for the $Q$ value function. Since $a$ is viewed as a random variable that follows the distribution $\mu$, $\nabla_{a} Q(s, a)$ can also be viewed as a random variable following a corresponding distribution $\mu_{grad}$. We assume that for any $x \in \mathbb{R}^d$, we have $\mu_{grad}(x) \leq b$. Intuitively, this is reasonable because before convergence, particles are likely to have different gradients from $Q$, causing them to move in different directions. This means that the gradient of $Q$ should not collapse to one point but instead be diverse. With this assumption, we have the following result.

% \begin{proposition}[continuous version]
%     Assume that the policies before and after an SVGD update both follow gaussian distributions, i.e., $\pi_1 \sim \mathcal{N}(\mu_1,\sigma_1^2), \pi_2 \sim \mathcal{N}(\mu_2, \sigma^2_2)$, then there exists a particle in the $c$-support of $\pi_1$ that moves out of it in $\pi_2$.
% \end{proposition}
As for the continuous setting, we typically assume that the policies before and after an SVGD update both follow Gaussian distributions, which is a widely used assumption. Formally, we assume that $\pi_N^0 \sim \mathcal{N}(\mu_1,\sigma_1^2), \pi_N^1 \sim \mathcal{N}(\mu_2, \sigma^2_2)$.

\begin{proof}
    We consider the gradient of an arbitrary particle $a$ sampled from $\pi_N^0$. We know that 
    \[
    h(a) = \mathbb{E}_{a_j \sim \pi_N^0} \left[ k(a_j, a) \nabla_{a_j} R(s, a_j) / \alpha + \nabla_{a_j} k(a_j, a) \right].
    \]
    Let us investigate the first term. We also know that
    \begin{align*}
        \nabla_{a_j} R(s, a_j) &= \nabla_{a_j} \log \pi_N^1(a_j|s) \\
        &= \nabla_{a_j} \log \left[\frac{1}{\sqrt{2\pi}\sigma_2} \exp(-\frac{(a_j-\mu_2)^2}{2\sigma^2_2})\right] \\
        &= \frac{-\frac{1}{\sqrt{2\pi}\sigma_2} \exp(-\frac{(a_j-\mu_2)^2}{2\sigma^2_2}) \frac{a_j-\mu_2}{\sigma^2_2}}{\frac{1}{\sqrt{2\pi}\sigma_2} \exp(-\frac{(a_j-\mu_2)^2}{2\sigma^2_2})} \\
        &= \frac{\mu_2-a_j}{\sigma^2_2}.
    \end{align*}
    Then, we calculate the second term:
    \begin{align*}
        \nabla_{a_j} k(a_j, a) &= k(a_j, a) \frac{a_j-a}{\sigma^2_1}.
    \end{align*}
    Combining both terms, we have
    \begin{align*}
        h(a) &= \mathbb{E}_{a_j \sim \pi_N^0} \left[k(a_j, a)(\frac{\mu_2-a_j}{\alpha \sigma^2_2} + \frac{a_j-a}{\sigma^2_1})\right] \\
        &= \mathbb{E}_{a_j \sim \pi_N^0} \left[k(a_j, a)\frac{\mu_1-a_j}{\alpha \sigma^2_2}\right] + \mathbb{E}_{a_j \sim \pi_N^0} \left[k(a_j, a)(\frac{\mu_2-\mu_1}{\alpha \sigma^2_2} + \frac{a_j-a}{\sigma^2_1})\right] \\
        &= \mathbb{E}_{a_j \sim \pi_N^0} \left[k(a_j, a)(\frac{\mu_2-\mu_1}{\alpha \sigma^2_2} + \frac{a_j-a}{\sigma^2_1})\right].
    \end{align*}
    The last equation above is because of the symmetry of Gaussian distributions. We now focus on the first dimension of particles. Without loss of generality, we assume that $\mu_{1,1} \leq \mu_{2,1}$. From the $\epsilon$-support of $\pi_N^0$, which is a closed region, we choose $a$ with the smallest value in the first dimension, i.e., $(a-x)_1 \leq 0$ for any $x \in \text{supp}(\pi_N^0)$. This indicates that the first dimension of $h(a)$ is strictly greater than zero, which means that the updated particle of $a$ is out of the $\epsilon$-support of $\pi_N^0$.
\end{proof}

% Empirically, during the training process, we can record the Wasserstein distance between the two distributions, showing that it is constantly declining. However, since we mainly deal with particles, the KL divergence can be hard to compute. We want to verify in experiments that the KL divergence is truly a large value when supports are different.

\section{Experimental Details}
\label{sec_exp_details}

\subsection{Offline RL Evaluation Details}
\noindent \textbf{Environments, tasks, and datasets.}  \
In the offline setting, VGF is evaluated on different kinds of datasets from various environments.

For MuJoCo environments, we have the following datasets.
\begin{itemize}[leftmargin=1.5em, labelsep=0.5em]
\item \texttt{halfcheetah/hopper/walker2d-m} (medium): Collected by a policy with moderate performance, typically reaching around one-third of expert returns. These datasets represent structured but suboptimal behavior.

\item \texttt{halfcheetah/hopper/walker2d-m-r} (medium-replay): Contains the replay buffer of the mediocre SAC policy. It includes a wide range of off-policy transitions, many of which are suboptimal or noisy.

\item \texttt{halfcheetah/hopper/walker2d-m-e} (medium-expert): A 50-50 mixture of medium and expert trajectories. These datasets are designed to test whether algorithms can leverage near-optimal data when it is partially present.
\end{itemize}

The AntMaze environments involve a quadruped ant navigating through a 2D maze using sparse goal-based rewards. The agent has a 29-dimensional state space and an 8-dimensional action space, corresponding to joint positions, velocities, and target location encoding. The tasks are particularly challenging due to long-horizon planning and sparse supervision.

\begin{itemize}[leftmargin=1.5em, labelsep=0.5em]

\item \texttt{antmaze-u} (umaze): A small maze where the agent must reach a fixed goal location using sparse rewards. The environment is relatively easy due to short trajectories.

\item \texttt{antmaze-u-d} (umaze-diverse): Similar to \texttt{umaze}, but with broader trajectory diversity collected from random exploration.

\item \texttt{antmaze-m-p} (medium-play): A medium-sized maze where data is collected via a play policy. The task is harder due to longer horizons and sparse goal rewards.

\item \texttt{antmaze-m-d} (medium-diverse): Features more diverse and noisy behavior than \texttt{medium-play}, increasing exploration coverage but decreasing consistency.

\item \texttt{antmaze-l-p} (large-play): A large maze with random play data. The agent must navigate long distances, making the task especially difficult under sparse reward signals.

\item \texttt{antmaze-l-d} (large-diverse): Similar to \texttt{large-play}, but with broader and more varied behavior. It is one of the most challenging offline datasets due to the size of the environment and variability in data.
\end{itemize}

OGBench not only substantially extends the original AntMaze environment provided by the D4RL benchmark, but also introduces more challenging tasks, such as humanoid control and object manipulation. We elaborate on our selected environments below.

\begin{itemize}[leftmargin=1.5em, labelsep=0.5em]

\item \texttt{antmaze-giant-navigate}: An 8-DoF ant agent needs to reach a goal location in a 2-D maze, the size of which is substantially larger than those of D4RL datasets.

\item \texttt{humanoid-medium/large-navigate}: This task involves full-body control of a 21-DoF Humanoid agent, which requires long-horizon reasoning.

\item \texttt{antsoccer-arena-navigate}: This task involves controlling an ant agent to dribble a soccer ball. The agent is required to approach the ball and dribble it to random locations in an arena.

\item \texttt{cube-single/double-play}: This task involves pick-and-place manipulation of several different cubes. The agent is required to complete tasks by moving, stacking, swapping, or permuting cubes. The options "single" and "double" refer to the number of cubes.

\item \texttt{scene-play}: The agent's goal to manipulate the two buttons to determine the statuses of a window and a drawer to finish a specific task, which requires sequential reasoning.

\item \texttt{puzzle-3$\times$3/4$\times$4-play}: The goal is to solve a "light-out" puzzle with a robot arm. At each step, the agent can press a button to change the states of the pressed button and its neighbors. We select $3 \times 3$ and $4 \times 4$ as the puzzle size.
\end{itemize}

\begin{figure}[t]
\centering
% \vspace{-0.3cm}
\includegraphics[width=0.97\textwidth]{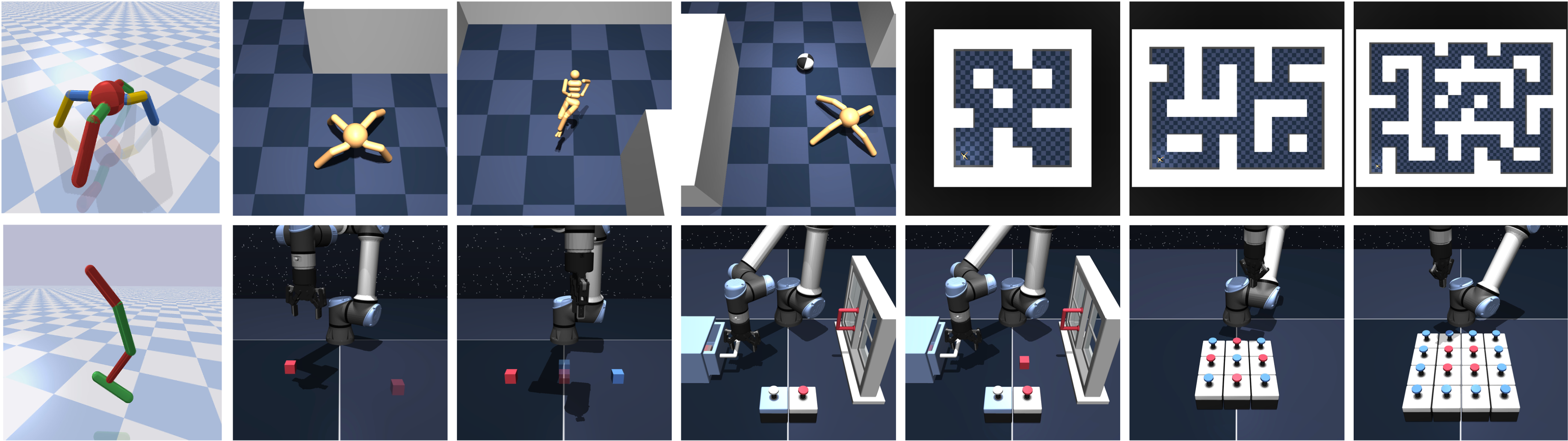}
% \vspace{-0.3cm}
\caption{Visualization of offline RL tasks.} 
% \label{fig_off2on_result}
% \vspace{-0.2cm}
\end{figure}

\noindent \textbf{Network Architecture and hyperparameters.} \
In this part, we provide details of the hyperparameters used in our experiments on D4RL datasets and OGBench. Since most of the hyperparameters remain the same for all of our experiments, we list them in \Cref{table_default_hyp} and demonstrate our choice of the important ones in \Cref{tab_hyperparameters_d4rl} and \Cref{tab_hyperparameters_ogbench}. Specifically, we regard VGF step size $\epsilon$ and training VGF steps $L_{\mathrm{train}}$ as the important hyperparameters.

We use a 4-layer MLP with 512 hidden units and Adam optimizer~\citep{adam} with a learning rate of $3\times10^{-4}$ for both the behavior cloning network and the critic network in all tasks. We also use a target network with soft update weight $5\times10^{-3}$ for critic update. We ran VGF for $10^6$ gradient steps for all our experiments with batch size 256. We set BC flow steps to 10. In addition, we set particle number $N$ to 5 because increasing $N$ will bring about extra computational costs but hardly any improvement in performance. Our results are averaged across 5 seeds for all our experiments, with standard deviation reported. In \Cref{table:full_ogbench_results}, we further provide detailed results on each of the $5\times9=45$ tasks in OGBench, which further verifies VGF's efficacy across most of the environments.

\subsection{RLHF Experiment Details}
We evaluate on the TL;DR Summarize corpus~\citep{stiennon2020learning} and the Anthropic Helpful \& Harmless (HH) corpus~\citep{bai2022}. To reduce degenerate generations that miss an eos token, we filter out overly long prompts before training and evaluation: we discard prompts longer than 448 tokens for TL;DR and 348 tokens for HH (lengths measured after tokenization with the base model's tokenizer). For TL;DR we use the dedicated SFT split provided by the dataset.
Unless otherwise noted, both the reward model and the policy are initialized from the same SFT checkpoint. For Pythia models we train the SFT stage for 2 epochs with an initial learning rate of $\times10^{-5}$ on both summarization and dialogue tasks. 
We train the reward model for 1 epoch with an initial learning rate of $1\times10^{-5}$ using the same train splits as the policy initialization (TL;DR references or HH chosen responses).
For both SFT and reward model training we use cosine learning-rate decay with a warm-up ratio of 0.03. All other implementation details are kept identical across tasks unless explicitly stated.

\begin{table}[t]
\caption{
\footnotesize
Default Hyperparameters for VGF.
}
\vspace{5pt}
\label{table_default_hyp}
\begin{center}
\scalebox{0.93}
{
\begin{tabular}{l|l}
    \toprule
    Hyperparameter & Value \\
    \midrule
    Critic learning rate & $0.0003$\\
    Actor learning rate & $0.0003$\\
    Gradient steps & $1000000$ \\
    Batch size & $256$ \\
    Critic Hidden dimensions & $[512, 512, 512, 512]$ \\
    Discount factor $\gamma$ & $0.99$ (default), $0.995$ (antmaze-giant, humanoidmaze, antsoccer)   \\
    BC flow steps & $10$ \\
    Q value Aggregation & $\mathrm{min}$ (D4RL and OGBench antmaze), $\mathrm{mean}$ (OGBench others) \\
    Particle number $N$ & $5$ \\
    Train particle select & $\mathrm{mean}$ \\
    Eval particle select & $\mathrm{max}$ \\
    Test VGF steps $L_{\mathrm{test}}$ & $[0, 1, 2, 3]$ \\
    \bottomrule
\end{tabular}
}
\end{center}
\end{table}

\begin{table}[t]
\centering
\caption{Hyperparameter selection of VGF in D4RL datasets.}
\label{tab_hyperparameters_d4rl}
\vspace{5pt}
\resizebox{0.7\linewidth}{!}{
\begin{tabular}{lcc} 
\toprule
Env                          & Train VGF steps $L_{\mathrm{train}}$ & VGF step size $\epsilon$  \\ 
\midrule
halfcheetah-medium-v2        & 3 & 0.05     \\
hopper-medium-v2             & 1 & 0.05      \\
walker2d-medium-v2           & 1 & 0.05   \\
halfcheetah-medium-replay-v2 & 3 & 0.05    \\
hopper-medium-replay-v2      & 1 & 0.05      \\
walker2d-medium-replay-v2    & 1 & 0.05      \\
halfcheetah-medium-expert-v2 & 3 & 0.05   \\
hopper-medium-expert-v2      & 1 & 0.05      \\
walker2d-medium-expert-v2    & 1 & 0.05   \\
\midrule
antmaze-umaze-v2             & 5 & 0.2    \\
antmaze-umaze-diverse-v2     & 5 & 0.2   \\
antmaze-medium-play-v2       & 5 & 0.2    \\
antmaze-medium-diverse-v2    & 5 & 0.1  \\
antmaze-large-play-v2        & 5 & 0.2    \\
antmaze-large-diverse-v2     & 5 & 0.2   \\
\bottomrule
\end{tabular}
}
\end{table}

\begin{table}[t]
\centering
\caption{Hyperparameter selection of VGF in OGBench datasets (offline and offline-to-online).}
\label{tab_hyperparameters_ogbench}
\vspace{5pt}
\resizebox{0.85\linewidth}{!}{
\begin{tabular}{lcc} 
\toprule
Env                          & Train VGF steps $L_{\mathrm{train}}$ & VGF step size $\epsilon$  \\ 
\midrule
antmaze-large-navigate-singletask-v0             & 5 & 0.1       \\
antmaze-giant-navigate-singletask-v0             & 5 & 0.1       \\
humanoidmaze-medium-navigate-singletask-v0         & 1 & 0.05   \\
humanoidmaze-large-navigate-singletask-v0 & 1 & 0.05    \\
antsoccer-arena-navigate-singletask-v0      & 2 & 0.05     \\
\midrule
cube-single-play-singletask-v0             & 1 & 0.05    \\
cube-double-play-singletask-v0     & 1 & 0.05    \\
scene-play-singletask-v0       & 1 & 0.1   \\
puzzle-3x3-play-singletask-v0    & 5 & 0.1  \\
puzzle-4x4-play-singletask-v0       & 5 & 0.1     \\
\bottomrule
\end{tabular}
}
\end{table}

% \begin{table}[th]
% \centering
% \caption{UNIVR hyperparameters in offline-to-online RL.}
% \label{tab_hyperparameters_offline2online}
% \vspace{5pt}
% \resizebox{0.45\linewidth}{!}{
% \begin{tabular}{l|rr} 
% \toprule
% Env                          & $\alpha$ & $\lambda$ \\ 
% \midrule
% pen-cloned-v1                & 1.0  & 0.01    \\
% door-cloned-v1                & 2.0  & 0.01    \\
% hammer-cloned-v1                & 2.0  & 0.01  \\
% relocate-cloned-v1                & 2.0  & 0.01   \\
% \midrule
% antmaze-medium-play-v2       & 0.5  & 0.01  \\
% antmaze-medium-diverse-v2    & 0.5  & 0.01 \\
% antmaze-large-play-v2        & 0.5  & 0.01  \\
% antmaze-large-diverse-v2     & 0.5  & 0.01  \\
% \bottomrule
% \end{tabular}
% }
% \end{table}

% DEFAULT TASK PERF

\setlength{\tabcolsep}{3pt}
\begin{table*}[t!]
\vspace{-10pt}

\caption{
\footnotesize
 {OGBench results (all tasks).} \methodname{} performs comparable or superior to the baselines on most tasks. \texttt{(*)} denotes the default task per environment~\citep{park2025flow}.}
\label{table:full_ogbench_results}
\centering
\vspace{-0.1cm}
\resizebox{0.93\linewidth}{!}
{
\begin{threeparttable}
\begin{tabular}{l ccc ccc | c}

\toprule

\multicolumn{1}{c}{} & \multicolumn{3}{c}{\texttt{Gaussian Policy}} & \multicolumn{3}{c}{\texttt{Diffusion/Flow Policy}} & \multicolumn{1}{c}{\texttt{w/o Policy (Ours)}} \\
\cmidrule(lr){2-4} \cmidrule(lr){5-7} \cmidrule(lr){8-8}
\texttt{Environment (5 tasks each)} & \texttt{BC} & \texttt{IQL} & \texttt{ReBRAC} & \texttt{FBRAC} &\texttt{IDQL} & \texttt{FQL} &\texttt{\methodname{}} \\

\midrule
% \texttt{antmaze-large-task1(*)}  &$0$ {\tiny $\pm 0$}& $91$ {\tiny $\pm 10$}  & $93$ {\tiny $\pm 2$} & $80$ {\tiny $\pm 8$}& $85$ {\tiny $\pm 4$}&  {$94$} {\tiny $\pm 4$} &  {$94$} {\tiny $\pm 4$}\\
% \texttt{antmaze-large-task2}  &$6$ {\tiny $\pm 3$}&  {$88$} {\tiny $\pm 4$}  & $79$ {\tiny $\pm 5$} & $57$ {\tiny $\pm 10$}& $64$ {\tiny $\pm 7$}& $82$ {\tiny $\pm 5$} & $82$ {\tiny $\pm 5$}\\
% \texttt{antmaze-large-task3}  &$29$ {\tiny $\pm 5$}& $51$ {\tiny $\pm 18$}  & $88$ {\tiny $\pm 10$} & $93$ {\tiny $\pm 3$}& $95$ {\tiny $\pm 2$}&  {$96$} {\tiny $\pm 4$} &  {$96$} {\tiny $\pm 4$}\\
% \texttt{antmaze-large-task4}  &$8$ {\tiny $\pm 3$}& $84$ {\tiny $\pm 7$}  & $91$ {\tiny $\pm 2$} & $80$ {\tiny $\pm 4$}& $86$ {\tiny $\pm 5$}&  {$90$} {\tiny $\pm 8$} &  {$90$} {\tiny $\pm 8$}\\
% \texttt{antmaze-large-task5}  &$10$ {\tiny $\pm 3$}& $90$ {\tiny $\pm 2$}  &  {$95$} {\tiny $\pm 0$} & $83$ {\tiny $\pm 4$}& $87$ {\tiny $\pm 5$}&  {$95$} {\tiny $\pm 4$} &  {$95$} {\tiny $\pm 4$}\\
% \midrule
\texttt{antmaze-giant-task1(*)}   &$0$ {\tiny $\pm 0$} & $0$ {\tiny $\pm 0$} & $27$ {\tiny $\pm 22$}  & $0$ {\tiny $\pm 1$} & $0$ {\tiny $\pm 0$} & $4$ {\tiny $\pm 5$} & {$0$} {\tiny $\pm 0$}\\
\texttt{antmaze-giant-task2}   &$0$ {\tiny $\pm 0$} & $1$ {\tiny $\pm 1$} & $16$ {\tiny $\pm 17$}  & $4$ {\tiny $\pm 7$} & $0$ {\tiny $\pm 0$} &  {$9$} {\tiny $\pm 7$}& $9$ {\tiny $\pm 3$}\\
\texttt{antmaze-giant-task3}   &$0$ {\tiny $\pm 0$} & $0$ {\tiny $\pm 0$} & $34$ {\tiny $\pm 22$}  & $0$ {\tiny $\pm 0$} & $0$ {\tiny $\pm 0$} & $0$ {\tiny $\pm 1$} & $0$ {\tiny $\pm 0$}\\
\texttt{antmaze-giant-task4}   &$0$ {\tiny $\pm 0$} & $0$ {\tiny $\pm 0$} & $5$ {\tiny $\pm 12$}  & $9$ {\tiny $\pm 4$} &  {$0$} {\tiny $\pm 0$} & $14$ {\tiny $\pm 23$}& $0$ {\tiny $\pm 0$}\\
\texttt{antmaze-giant-task5}   &$1$ {\tiny $\pm 1$} & $19$ {\tiny $\pm 7$} & $49$ {\tiny $\pm 22$}  & $6$ {\tiny $\pm 10$} & $0$ {\tiny $\pm 1$} & $16$ {\tiny $\pm 28$}& {$6$} {\tiny $\pm 2$}\\
\midrule
\texttt{hmmaze-medium-task1(*)}  &$1$ {\tiny $\pm 0$} &$32$ {\tiny $\pm 7$} & $16$ {\tiny $\pm 9$}  & $25$ {\tiny $\pm 8$} & $1$ {\tiny $\pm 1$}& $19$ {\tiny $\pm 12$}&  {$86$} {\tiny $\pm 1$}\\
\texttt{hmmaze-medium-task2}  &$1$ {\tiny $\pm 0$} &$41$ {\tiny $\pm 9$} & $18$ {\tiny $\pm 16$}  & $76$ {\tiny $\pm 10$} & $1$ {\tiny $\pm 1$}& $94$ {\tiny $\pm 3$}&  {$92$} {\tiny $\pm 2$}\\
\texttt{hmmaze-medium-task3}  &$6$ {\tiny $\pm 2$} &$25$ {\tiny $\pm 5$}& $36$ {\tiny $\pm 13$}  & $27$ {\tiny $\pm 11$} & $0$ {\tiny $\pm 1$}& $74$ {\tiny $\pm 18$}&  {$87$} {\tiny $\pm 2$}\\
\texttt{hmmaze-medium-task4}  &$0$ {\tiny $\pm 0$} &$0$ {\tiny $\pm 1$}& $15$ {\tiny $\pm 16$}  & $1$ {\tiny $\pm 2$} & $1$ {\tiny $\pm 1$}&  {$3$} {\tiny $\pm 4$}&  {$0$} {\tiny $\pm 0$}\\
\texttt{hmmaze-medium-task5}  &$2$ {\tiny $\pm 1$} &$66$ {\tiny $\pm 4$}& $24$ {\tiny $\pm 20$}  & $63$ {\tiny $\pm 9$} & $1$ {\tiny $\pm 1$}&  {$97$} {\tiny $\pm 2$}&  {$97$} {\tiny $\pm 0$}\\
\midrule
\texttt{hmmaze-large-task1(*)}   &$0$ {\tiny $\pm 0$}& $3$ {\tiny $\pm 1$} & $2$ {\tiny $\pm 1$} & $0$ {\tiny $\pm 1$} & $0$ {\tiny $\pm 0$}& $7$ {\tiny $\pm 6$}&  {$26$} {\tiny $\pm 4$}\\
\texttt{hmmaze-large-task2}   &$0$ {\tiny $\pm 0$}&$0$ {\tiny $\pm 0$}& $0$ {\tiny $\pm 0$}  & $0$ {\tiny $\pm 0$} & $0$ {\tiny $\pm 0$}& $0$ {\tiny $\pm 0$}&  {$2$} {\tiny $\pm 1$}\\
\texttt{hmmaze-large-task3}   &$1$ {\tiny $\pm 1$}&$7$ {\tiny $\pm 3$}& $8$ {\tiny $\pm 4$}  & $10$ {\tiny $\pm 2$} & $3$ {\tiny $\pm 1$}& $11$ {\tiny $\pm 7$}&  {$12$} {\tiny $\pm 6$}\\
\texttt{hmmaze-large-task4}   &$1$ {\tiny $\pm 0$}&$1$ {\tiny $\pm 0$}& $1$ {\tiny $\pm 1$}  & $0$ {\tiny $\pm 0$} & $0$ {\tiny $\pm 0$}&  {$2$} {\tiny $\pm 3$}& $2$ {\tiny $\pm 1$}\\
\texttt{hmmaze-large-task5}   &$0$ {\tiny $\pm 1$}&$1$ {\tiny $\pm 1$}& $2$ {\tiny $\pm 2$}  & $1$ {\tiny $\pm 1$} & $0$ {\tiny $\pm 0$}& $1$ {\tiny $\pm 3$}&  {$31$} {\tiny $\pm 5$}\\
\midrule
\texttt{antsoccer-arena-task1}  &$2$ {\tiny $\pm 1$}&$14$ {\tiny $\pm 5$} & $0$ {\tiny $\pm 0$}  &  {$17$} {\tiny $\pm 3$} & $44$ {\tiny $\pm 12$}& $77$ {\tiny $\pm 4$}& $76$ {\tiny $\pm 3$}\\
\texttt{antsoccer-arena-task2}  &$2$ {\tiny $\pm 2$}&$17$ {\tiny $\pm 7$} & $0$ {\tiny $\pm 1$}  &  $8$ {\tiny $\pm 2$} & $15$ {\tiny $\pm 12$}& $88$ {\tiny $\pm 3$}&  {$75$} {\tiny $\pm 2$}\\
\texttt{antsoccer-arena-task3}  &$0$ {\tiny $\pm 0$}&$6$ {\tiny $\pm 4$} & $0$ {\tiny $\pm 0$}  &  {$16$} {\tiny $\pm 3$} & $0$ {\tiny $\pm 0$}& $61$ {\tiny $\pm 6$}& $41$ {\tiny $\pm 9$}\\
\texttt{antsoccer-arena-task4(*)}  &$1$ {\tiny $\pm 0$}&$3$ {\tiny $\pm 2$} & $0$ {\tiny $\pm 0$}  &  {$24$} {\tiny $\pm 4$} & $0$ {\tiny $\pm 1$}& $39$ {\tiny $\pm 6$}& $58$ {\tiny $\pm 8$}\\
\texttt{antsoccer-arena-task5}  &$0$ {\tiny $\pm 0$}&$2$ {\tiny $\pm 2$} & $0$ {\tiny $\pm 0$}  &  {$15$} {\tiny $\pm 4$} & $0$ {\tiny $\pm 0$}& $36$ {\tiny $\pm 9$}& $64$ {\tiny $\pm 8$}\\
\midrule
\texttt{cube-single-task1}  &$10$ {\tiny $\pm 5$}&$88$ {\tiny $\pm 3$}& $89$ {\tiny $\pm 5$}  & $73$ {\tiny $\pm 33$}& $95$ {\tiny $\pm 2$}& $97$ {\tiny $\pm 2$}&  {$98$} {\tiny $\pm 2$}\\
\texttt{cube-single-task2(*)}  &$3$ {\tiny $\pm 1$}&$85$ {\tiny $\pm 8$}& $92$ {\tiny $\pm 4$}  &  {$83$} {\tiny $\pm 13$}& $96$ {\tiny $\pm 2$}& $97$ {\tiny $\pm 2$}&  {$100$} {\tiny $\pm 0$}\\
\texttt{cube-single-task3}  &$9$ {\tiny $\pm 3$}&$91$ {\tiny $\pm 5$}& $93$ {\tiny $\pm 3$}  &  {$82$} {\tiny $\pm 12$}& $99$ {\tiny $\pm 1$}& $98$ {\tiny $\pm 2$}& $100$ {\tiny $\pm 0$}\\
\texttt{cube-single-task4}  &$2$ {\tiny $\pm 1$}&$73$ {\tiny $\pm 6$}& $92$ {\tiny $\pm 3$}  & $79$ {\tiny $\pm 20$}& $93$ {\tiny $\pm 4$}& $94$ {\tiny $\pm 3$}&  {$94$} {\tiny $\pm 3$}\\
\texttt{cube-single-task5}  &$3$ {\tiny $\pm 3$}&$78$ {\tiny $\pm 9$}& $87$ {\tiny $\pm 8$}  & $76$ {\tiny $\pm 33$}& $90$ {\tiny $\pm 6$}& $93$ {\tiny $\pm 3$}&  {$92$} {\tiny $\pm 4$}\\
\midrule
\texttt{cube-double-task1}    &$8$ {\tiny $\pm 3$}&$27$ {\tiny $\pm 5$}& $45$ {\tiny $\pm 6$} &  {$47$} {\tiny $\pm 11$} & $39$ {\tiny $\pm 19$}& $61$ {\tiny $\pm 9$}& $95$ {\tiny $\pm 6$}\\
\texttt{cube-double-task2(*)}    &$0$ {\tiny $\pm 0$}&$1$ {\tiny $\pm 1$}& $7$ {\tiny $\pm 3$} & $22$ {\tiny $\pm 12$} & $16$ {\tiny $\pm 10$}& $36$ {\tiny $\pm 6$}&  {$78$} {\tiny $\pm 10$}\\
\texttt{cube-double-task3}    &$0$ {\tiny $\pm 0$}&$0$ {\tiny $\pm 0$}& $4$ {\tiny $\pm 1$} & $4$ {\tiny $\pm 2$} & $17$ {\tiny $\pm 8$}& $22$ {\tiny $\pm 5$}&  {$66$} {\tiny $\pm 8$}\\
\texttt{cube-double-task4}    &$0$ {\tiny $\pm 0$}&$0$ {\tiny $\pm 0$}& $1$ {\tiny $\pm 1$} & $0$ {\tiny $\pm 1$} & $0$ {\tiny $\pm 1$}& $5$ {\tiny $\pm 2$}&  {$31$} {\tiny $\pm 4$}\\
\texttt{cube-double-task5}    &$0$ {\tiny $\pm 0$}&$4$ {\tiny $\pm 3$}& $4$ {\tiny $\pm 2$} & $2$ {\tiny $\pm 2$} & $1$ {\tiny $\pm 1$}& $19$ {\tiny $\pm 10$}&  {$78$} {\tiny $\pm 11$}\\
\midrule
\texttt{scene-task1} &$19$ {\tiny $\pm 6$}&$94$ {\tiny $\pm 3$}& $95$ {\tiny $\pm 2$}  & $96$ {\tiny $\pm 8$} &  {$100$} {\tiny $\pm 0$}& $100$ {\tiny $\pm 0$}&  {$100$} {\tiny $\pm 0$}\\
\texttt{scene-task2(*)} &$1$ {\tiny $\pm 1$}&$12$ {\tiny $\pm 3$}& $50$ {\tiny $\pm 13$}  &  {$46$} {\tiny $\pm 10$} & $33$ {\tiny $\pm 14$}& $76$ {\tiny $\pm 9$}& $96$ {\tiny $\pm 2$}\\
\texttt{scene-task3} &$1$ {\tiny $\pm 1$}&$32$ {\tiny $\pm 7$}& $55$ {\tiny $\pm 16$}  & $78$ {\tiny $\pm 14$} &  {$94$} {\tiny $\pm 4$}& $98$ {\tiny $\pm 1$}&  {$98$} {\tiny $\pm 2$}\\
\texttt{scene-task4} &$2$ {\tiny $\pm 2$}&$0$ {\tiny $\pm 1$}& $3$ {\tiny $\pm 3$}  & $4$ {\tiny $\pm 4$} & $4$ {\tiny $\pm 3$}&  {$5$} {\tiny $\pm 1$}& $2$ {\tiny $\pm 2$}\\
\texttt{scene-task5} &$0$ {\tiny $\pm 0$}&$0$ {\tiny $\pm 0$}& $0$ {\tiny $\pm 0$}  & $0$ {\tiny $\pm 0$} & $0$ {\tiny $\pm 0$}& $0$ {\tiny $\pm 0$}& $2$ {\tiny $\pm 1$}\\
\midrule
\texttt{puzzle-3x3-task1}   &$5$ {\tiny $\pm 2$}&$33$ {\tiny $\pm 6$}&  {$97$} {\tiny $\pm 4$}& $63$ {\tiny $\pm 19$}& $52$ {\tiny $\pm 12$}& $90$ {\tiny $\pm 4$}& $100$ {\tiny $\pm 0$}\\
\texttt{puzzle-3x3-task2}   &$1$ {\tiny $\pm 1$}&$4$ {\tiny $\pm 3$}& $1$ {\tiny $\pm 1$}& $2$ {\tiny $\pm 2$}& $0$ {\tiny $\pm 1$}& $16$ {\tiny $\pm 5$}&  {$71$} {\tiny $\pm 6$}\\
\texttt{puzzle-3x3-task3}   &$1$ {\tiny $\pm 1$}&$3$ {\tiny $\pm 2$}& $3$ {\tiny $\pm 1$}& $1$ {\tiny $\pm 1$}& $0$ {\tiny $\pm 0$}& $10$ {\tiny $\pm 3$}&  {$62$} {\tiny $\pm 7$}\\
\texttt{puzzle-3x3-task4(*)}   &$1$ {\tiny $\pm 1$}&$2$ {\tiny $\pm 1$}& $2$ {\tiny $\pm 1$}& $2$ {\tiny $\pm 2$}& $0$ {\tiny $\pm 0$}& $16$ {\tiny $\pm 5$}&  {$65$} {\tiny $\pm 6$}\\
\texttt{puzzle-3x3-task5}   &$1$ {\tiny $\pm 0$}&$3$ {\tiny $\pm 2$}& $5$ {\tiny $\pm 3$}& $2$ {\tiny $\pm 2$}& $0$ {\tiny $\pm 0$}& $16$ {\tiny $\pm 3$}&  {$76$} {\tiny $\pm 6$}\\
\midrule
\texttt{puzzle-4x4-task1}   &$1$ {\tiny $\pm 1$}&$12$ {\tiny $\pm 2$} & $26$ {\tiny $\pm 4$} & $32$ {\tiny $\pm 9$}& $48$ {\tiny $\pm 5$}& $34$ {\tiny $\pm 8$}& $85$ {\tiny $\pm 8$ }\\
\texttt{puzzle-4x4-task2}   &$0$ {\tiny $\pm 0$}&$7$ {\tiny $\pm 4$} & $12$ {\tiny $\pm 4$} & $5$ {\tiny $\pm 3$}& $14$ {\tiny $\pm 5$}& $16$ {\tiny $\pm 5$}& $26$ {\tiny $\pm 3$ }\\
\texttt{puzzle-4x4-task3}   &$0$ {\tiny $\pm 0$}&$9$ {\tiny $\pm 3$} & $15$ {\tiny $\pm 3$} & $20$ {\tiny $\pm 10$}& $34$ {\tiny $\pm 5$}& $18$ {\tiny $\pm 5$}& $65$ {\tiny $\pm 6$ }\\
\texttt{puzzle-4x4-task4(*)}   &$0$ {\tiny $\pm 0$}&$5$ {\tiny $\pm 2$} & $10$ {\tiny $\pm 3$} & $5$ {\tiny $\pm 1$}& $26$ {\tiny $\pm 6$}& $11$ {\tiny $\pm 3$}& $33$ {\tiny $\pm 6$ }\\
\texttt{puzzle-4x4-task5}   &$0$ {\tiny $\pm 0$}&$4$ {\tiny $\pm 1$} & $7$ {\tiny $\pm 3$} & $4$ {\tiny $\pm 3$}& $24$ {\tiny $\pm 11$}& $7$ {\tiny $\pm 3$}& $18$ {\tiny $\pm 2$ }\\
\midrule
\end{tabular}
\end{threeparttable}
}
\end{table*}

\section{Psedocode of VGF}
\label{sec_code}
\begin{figure}[ht]
\centering
\begin{spacing}{0.9}
\begin{python}
import jax
import jax.numpy as jnp

def rbf_kernel(X, Y, sigma=None):
    """X: [B, n, d], Y: [B, m, d], returns K_XY: [B, n, m]"""
    X2 = jnp.sum(X * X, axis=-1, keepdims=True)                                 
    Y2 = jnp.sum(Y * Y, axis=-1, keepdims=True).transpose(0, 2, 1)              
    XY = jnp.matmul(X, Y.transpose(0, 2, 1))                                    
    dnorm2 = X2 + Y2 - 2.0 * XY                                                 
    dnorm2 = jnp.maximum(dnorm2, 0.0)
    if sigma is None:
        # median heuristic per batch
        h = jnp.median(dnorm2, axis=(1,2)) 
        h /= (2.0 * jnp.log(X.shape[1] + 1.0)) 
        sigma_val = jnp.sqrt(jnp.maximum(h, 1e-12))                            
        sigma_val = sigma_val[:, None, None]
    else:
        sigma_val = jnp.asarray(sigma)
        if sigma_val.ndim == 0:
            sigma_val = jnp.broadcast_to(sigma_val, (X.shape[0], 1, 1))
    gamma = 1.0 / (1e-6 + 2.0 * (sigma_val ** 2))
    K_XY = jnp.exp(-gamma * dnorm2)                                            
    return K_XY

class VGF:
    def __init__(self, q, alpha, optimizer):
        self.q = q
        self.alpha = alpha
        self.optim = optimizer
        self.opt_state = None

    def init(self, particles):
        self.opt_state = self.optim.init(particles)
        return particles, self.opt_state

    def phi(self, obs, particles):
        # obs: [B, D], particles: [B, N, D]
        # score terms
        def sum_q(action):
            obs_flatten = obs.reshape(-1, obs.shape[-1])                   
            action_flatten = action.reshape(-1, action.shape[-1])        
            qs = self.q(obs_flatten, action_flatten)
            q = jnp.mean(qs, axis=0)
            return jnp.sum(q)
        score = jax.grad(sum_q)(particles)                    
        # kernel terms
        particles_stop = jax.lax.stop_gradient(particles)
        K_xx = rbf_kernel(particles, particles_stop)        
        def sum_K(x):
            return jnp.sum(rbf_kernel(x, particles_stop))
        grad_q = jax.lax.stop_gradient(K_xx) @ score
        grad_K = -jax.grad(sum_K)(particles)                       
        phi_val = (grad_q / self.alpha + grad_K) / particles.shape[1]
        return phi_val      

    def step(self, obs, particles, opt_s):
        grads = self.phi(obs, particles)
        updates, new_opt_s = self.optim.update(-grads, opt_s, particles)
        new_particles = optax.apply_updates(particles, updates)
        return new_particles, new_opt_s

\end{python}
\end{spacing}
\caption{A simple implementation of the VGF process.}
\label{fig:code}
\end{figure}

\end{document}